\documentclass{article}

\PassOptionsToPackage{numbers, compress}{natbib}
\usepackage[preprint]{neurips_2026}

\usepackage{amssymb}
\usepackage[utf8]{inputenc} 
\usepackage[T1]{fontenc}    
\usepackage{hyperref}       
\usepackage{url}            
\usepackage{booktabs}       
\usepackage{amsfonts}       
\usepackage{nicefrac}       
\usepackage{microtype}      
\usepackage{xcolor}         
\usepackage{amsmath}
\usepackage{amssymb}
\usepackage{mathtools}
\usepackage{amsthm}
\usepackage{subfigure}
\usepackage{longtable}
\usepackage{booktabs} 
\usepackage{caption}
\usepackage{multirow}
\usepackage{bbm}
\usepackage{tcolorbox}
\usepackage{float}
\tcbuselibrary{skins, breakable}
\usepackage[capitalize,noabbrev]{cleveref}
\usepackage{bbding}
\theoremstyle{plain}
\newtheorem{theorem}{Theorem}[section]

\theoremstyle{definition}

\theoremstyle{remark}

\title{Boosting LLM Reasoning via Human-Inspired Reward Shaping}

\author{%
  Wenze Lin\thanks{Equal contribution} \\
  Qiuzhen College, Tsinghua University \\
  \And
  Zhen Yang\footnotemark[1] \\
  Purdue University\\
  \And
  Xitai Jiang\footnotemark[1] \\
  Qiuzhen College, Tsinghua University \\
  \And
  Xiaoteng Ma \\
  Mind Lab \\
  \And
  Gao Huang\thanks{Corresponding author: \texttt{gaohuang@tsinghua.edu.cn}} \\
  Tsinghua University \\
}
\begin{document}

\maketitle

\begin{abstract}

Reinforcement Learning with Verifiable Rewards (RLVR) has emerged as a promising paradigm for enhancing reasoning in Large Language Models (LLMs). However, existing reward formulations typically treat exploration and consolidation as a monolithic process, resulting in entangled stage-wise learning dynamics. This contradicts the natural learning behavior of human learners. In human learning, individuals adopt distinct behavioral patterns toward mastered versus unfamiliar problems. When confronting unmastered challenges, humans prioritize broad exploration to seek viable solutions. By contrast, for well-mastered problems, they focus instead on reasoning condensation and knowledge abstraction to distill concise underlying principles. Motivated by this gap, we introduce \textbf{T2T} (Thickening-to-Thinning), a dynamic reward framework inspired by human learning processes. Specifically, it implements a dual-phase mechanism: (1) On incorrect attempts, T2T incentivizes "thickening" to broaden the search space and explore novel solution paths; (2) Upon achieving correctness, it shifts to "thinning", imposing length penalties to discourage redundancy, thereby fostering model confidence and crystallizing reasoning capabilities. Extensive experiments on mathematical benchmarks (MATH-500, AIME, AMC) across 5 mainstream LLMs demonstrate that T2T significantly outperforms standard GRPO and recent baselines, achieving superior performance.

\end{abstract}

\section{Introduction}

Psychological studies reveal that when faced with unfamiliar and unsolved problems, humans instinctively engage in diverse trials to seek viable solutions \citep{schneider1977controlled,fitts1967human,Anderson1980CognitivePA}, a tendency particularly prominent in reasoning tasks \citep{wood1976role}. For instance, when encountering difficult problems in mathematical examinations, students often attempt a variety of approaches and fill their answer sheets with trial reasoning even without a clear final solution. In contrast, as a problem gradually becomes familiar and mastered, humans naturally condense their problem-solving procedures to refine internal knowledge \citep{shiffrin1977controlled2,fitts1967human,Anderson1980CognitivePA}. A typical example is the Pythagorean theorem: primary school students tend to re-derive the theorem every time they apply it shortly after learning it. As their knowledge accumulates and proficiency grows, however, they directly deploy the theorem without repeated derivation, since the knowledge has already been internally consolidated and refined. Taken together, human learning follows a clear two-stage pattern \citep{fitts1967human,kim2013integrated,Anderson1980CognitivePA}. For unmastered challenges, learning starts with an expansive exploration phase: humans try multiple approaches, explore alternative decompositions, and tolerate verbosity and redundancy as a necessary cost of discovery. As mastery deepens, learning transitions into a second phase—reasoning is summarized and abstracted, and unnecessary details are stripped away to form compact, precise representations for efficient retention and reuse.

\begin{figure*}[htbp] 
  \centering
  
  \begin{minipage}[t]{0.48\textwidth}
    \centering
    \includegraphics[width=1.0\linewidth]{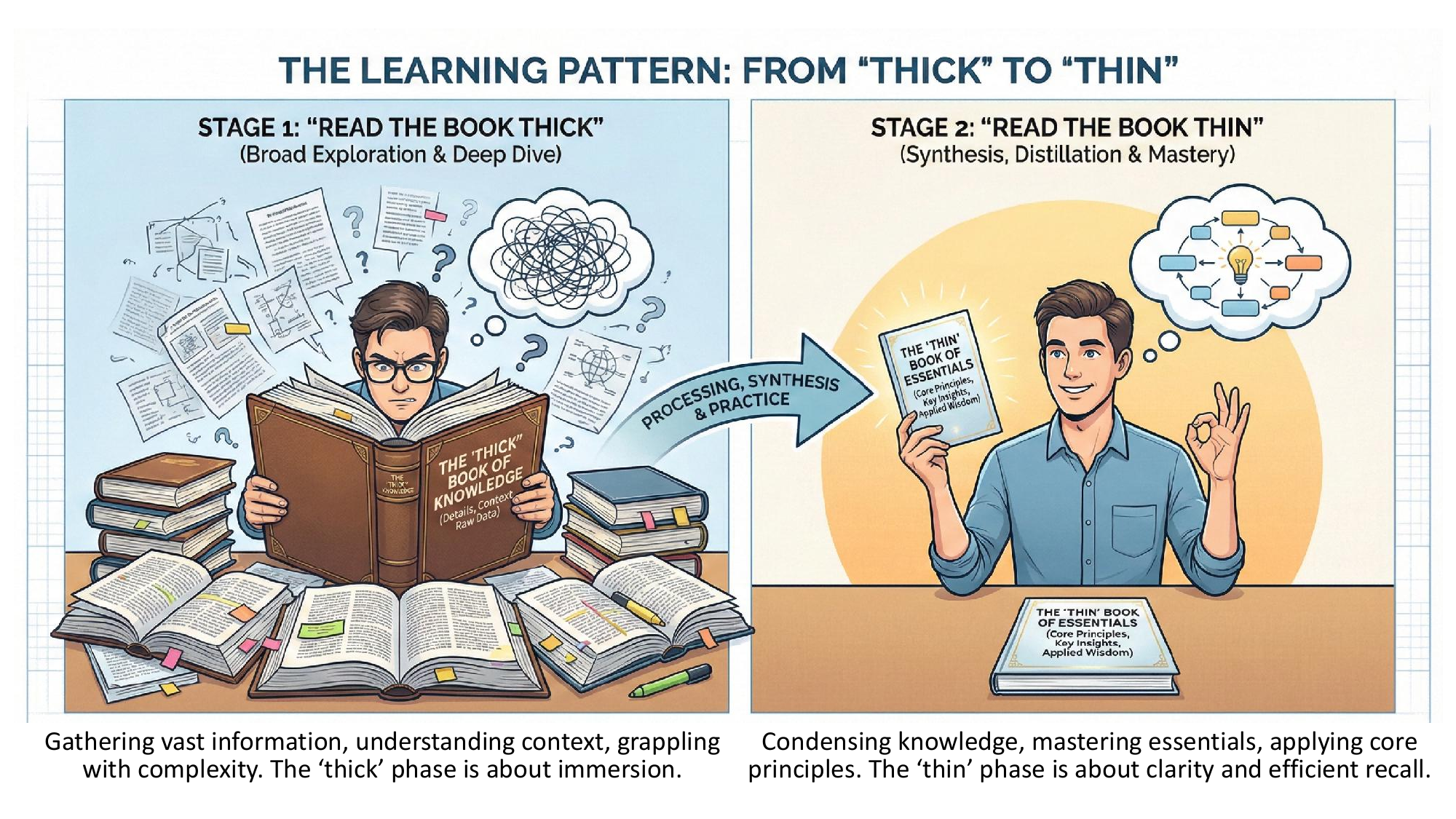}
    \caption{A two-stage learning pattern in human learning, where understanding is first expanded through exploration and later refined into concise and efficient knowledge.}
    \label{fig:t2t_paradigm}
  \end{minipage}
  \hfill 
  \begin{minipage}[t]{0.48\textwidth}
    \centering
    \includegraphics[width=1.0\linewidth]{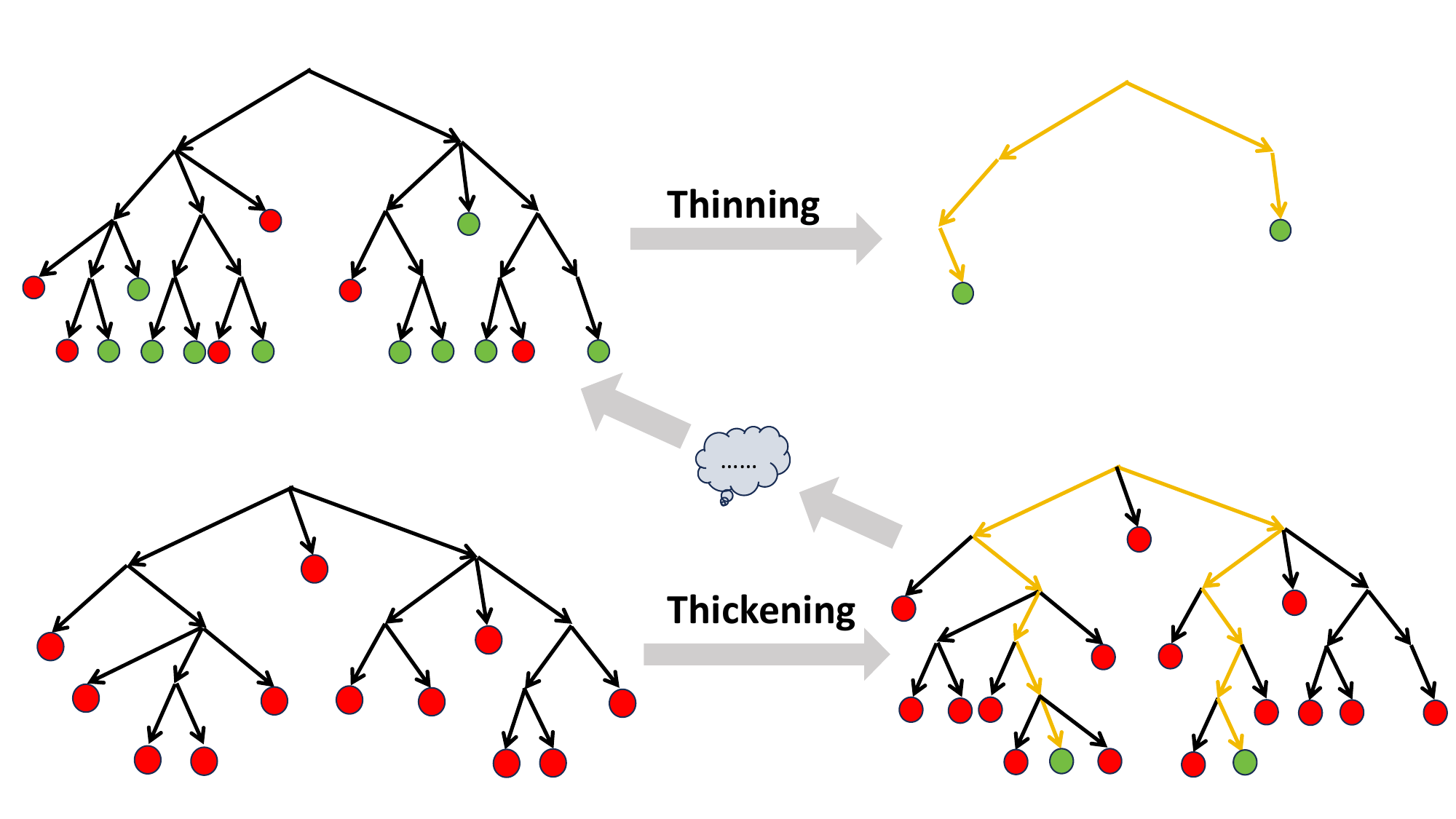}
    \caption{(1) On incorrect attempts, T2T incentivizes "thickening" to broaden the search space; (2) Upon correctness, it shifts to "thinning" to discourage redundancy, fostering model confidence.}
    \label{fig:t2t_mechanism}
  \end{minipage}
  
\end{figure*}

Reinforcement Learning with Verifiable Rewards (RLVR) has emerged as a practical paradigm for enhancing the reasoning capabilities of Large Language Models (LLMs), particularly in automatically verifiable domains such as mathematics and code \citep{DeepSeekR1,DeepSeekMath,jaech2024openai, trinh2024solving,yang2024qwen2}. However, existing RLVR methods misalign with natural human learning dynamics we mentioned before, failing to make distinction between exploration and condensation. While some existing RLVR extensions have attempted to differentiate between positive and negative samples, their optimization objectives are either focused on compressing reasoning trajectories or maintaining entropy stability \citep{NSR,EntroPIC,LASER,Tan2025GTPOAG}. The inherent learning dynamics between exploration for unmastered problems and condensation for mastered ones remain entangled without structural separation.



To address this gap, we propose Thickening-to-Thinning (T2T), a reward shaping method that explicitly embeds this stage-dependent learning dynamic into RLVR. Instead of redefining correctness, T2T modulates reward magnitude based on correctness, difficulty, and response length. It encourages expansive, exploratory reasoning when correctness is unachieved and favors concise, precise solutions once success is attained. What's more, T2T adopts a difficulty-aware incentive principle: for hard incorrect problems, T2T offers stronger encouragement for exploratory reasoning compared with easy failed cases; for simple solved problems, it imposes greater emphasis on concise reasoning than it does for correctly solved hard problems. Notably, T2T retains the simplicity of standard sequence-level RLVR: verification stays at the sequence level, no auxiliary models or token-level supervision are required. Consequently, T2T can be seamlessly integrated into existing RLVR pipelines and extended to various reinforcement learning variants, offering a minimal yet principled mechanism to align optimization objectives with the natural exploration-to-compression progression.

We conduct extensive experiments across five mainstream LLMs and four standard mathematical reasoning benchmarks. Empirical results demonstrate that T2T consistently achieves substantial gains over vanilla GRPO and a series of strong RLVR baselines. On the most capable model in our suite, Qwen3-14B, T2T improves AIME'25 Pass@64 from $51.5$ to $57.9$ ($+6.4$), while pushing AMC'23 Pass@64 from $94.6$ to $98.4$. Comparable improvements hold at smaller scales (e.g., Qwen3-4B AIME'24 Pass@64: $48.6 \to 56.3$) and across architectures (Llama-3.2-3B-Instruct AMC'23 Pass@64: $43.7 \to 49.1$), confirming T2T's strong cross-model transferability and broad applicability to diverse mathematical reasoning tasks.

\section{Background and Related Works}

\paragraph{Reinforcement Learning for Reasoning} 
Reinforcement learning (RL) has become standard for aligning language models with complex logical constraints \citep{ouyang2022training,jaech2024openai,trinh2024solving,yang2024qwen2}. Iterative training frameworks like STaR \citep{zelikman2022star} and ReST \citep{singh2023rest} utilize the model's own successful trajectories as bootstrap supervision. More recent policy optimization methods, such as PPO \citep{schulman2017proximal} and Group Relative Policy Optimization (GRPO) \citep{DeepSeekR1,DeepSeekMath}, directly optimize reasoning performance against outcome-based or process-based rewards \citep{lightman2023let}. However, standard RL objectives for reasoning typically treat the reasoning chain length as a byproduct or explicitly penalize verbosity to reduce costs \citep{dubois2024alpacafarm,LASER,sui2025stop,hou2025thinkprune,huang2026hapo}. In contrast, our T2T approach treats trajectory length as instrument to realize human learning dynamic, rather than a goal itself.

\paragraph{Variants and Extensions of RLVR.}
\label{app:related:passk_optimization}
A growing body of work has sought to address limitations of standard RLVR, including the mismatch between pass@1 improvements and large-$k$ behavior, entropy collapse, and excessive verbosity.
One line of work explicitly targets pass@k-style objectives or their surrogates~\citep{PKPO,bagirov2025bestofn,RSPO,chen2025passk,peng2025simko,RS-GRPO}.
Another line focuses on theoretical or statistical corrections to the optimization process itself, for example via entropy control~\cite{EntroPIC,li2025cure,jiang2025rethinking,tan2025gtpo,cheng2026reasoning}, structured pruning or brevity-promoting objectives~\citep{hou2025thinkprune, LASER}, or bias-aware alternatives to GRPO~\citep{DrGRPO, yang2026your}.
These approaches are primarily motivated by formal objective mismatch or optimization pathologies, and propose principled modifications to the learning objective or update rule. In a different vein, NSR~\cite{NSR} demonstrates that training solely on negative samples can still yield substantial capability improvements, highlighting the robustness of reward-driven learning signals.
Our work follows a different design philosophy.
Rather than deriving corrections from theoretical optimality criteria, we adopt a heuristic perspective inspired by human learning behavior, and modify the reward objective accordingly.
Since our proposed T2T method operates at the level of reward design, 
it is orthogonal to such framework-level extensions and can be applied alongside them 
without additional algorithmic overhead.

\paragraph{RLVR with Length Control.}
A series of work have incorporated length control mechanisms into RLVR ~\cite{LASER,sui2025stop,xiang2025just,hammoud2025train,hou2025thinkprune,huang2026hapo}. However, most prior works adopt explicit length penalties to compress inference overhead or reduce overthinking, aiming to maintain reasoning performance at lower computational cost. Their focus is all on efficient reasoning. In contrast, our method differs fundamentally from these approaches. For T2T, length modulation is not an optimization goal in itself, but an instrument to instantiate our human-inspired learning dynamics. Our primary objective is to boost the reasoning capability of LLMs by embedding the thickening-to-thinning learning paradigm into RLVR. We treat reasoning length as a dynamic exploration budget: it encourages the model to expand its reasoning search space on challenging unsolved problems, and drives concise, confident responses once the corresponding reasoning skill is fully mastered. In T2T, adaptive length modulation serves only as a supportive means to achieve stronger reasoning performance, rather than being pursued as an end goal.

\section{Preliminaries}
\subsection*{Reinforcement Learning with Verifiable Rewards}

Let $\mathcal{V}:\mathcal{Q}\times\mathcal{O}\to\{0,1\}$ denote a verifier that checks whether an output $o$ solves query $q$:
\begin{equation}
\mathcal{V}(q,o)=
\begin{cases}
1, & \text{if $o$ is verified as correct for $q$},\\
0, & \text{otherwise}.
\end{cases}
\end{equation}
Fix an interaction protocol (prompting, decoding strategy, and any hard constraints) and a model policy
$\pi_\theta(o\mid q)$ over outputs $o\in\mathcal{O}$ given a query $q$.
Define the single-sample success probability
\begin{equation}
p_\theta(q)
~:=~
\Pr_{o\sim\pi_\theta(\cdot\mid q)}[\,\mathcal{V}(q,o)=1\,].
\end{equation}

For clarity, we consider sequence-level reinforcement learning with verifiable rewards and use it to illustrate the algorithms throughout this paper.
Given a query $q \sim \mathcal{D}$ and a sequence-level output $o \sim \pi_\theta(\cdot | q)$,
the learning objective is to maximize the expected reward:
\begin{equation}
    \mathcal{J}(\pi_\theta)
    =
    \mathbb{E}_{q \sim \mathcal{D},\, o \sim \pi_\theta(\cdot \mid q)}
    \bigl[ R(q, o) \bigr].
\end{equation}
In GRPO and related RLVR methods, the reward function
$R(q,o)$ is typically defined as a binary indicator reflecting the
correctness of the generated output.
Under this definition, the objective reduces to maximizing the model’s
success probability under single-sample decoding:
\begin{equation}
\mathcal{J}(\pi_\theta)
=
\mathbb{E}_{q \sim \mathcal{D}}
\Big[
\mathbb{E}_{o \sim \pi_\theta(\cdot \mid q)} \big[ \mathcal{V}(q,o) \big]
\Big]
=
\mathbb{E}_{q \sim \mathcal{D}} \big[ p_\theta(q) \big],
\end{equation}
where $p_\theta(q) \triangleq \Pr_{o \sim \pi_\theta(\cdot \mid q)}\!\big(\mathcal{V}(q,o)=1\big)$
denotes the probability that the model produces a verifier-accepted output
for query $q$.

Using the policy gradient theorem, the gradient of the objective can be written as
\begin{equation}
    \nabla_\theta \mathcal{J}(\pi_\theta)
    =
    \mathbb{E}_{q \sim \mathcal{D},\, o \sim \pi_\theta(\cdot \mid q)}
    \Bigl[
        A^{\pi_\theta}(q,o)\,
        \nabla_\theta \log \pi_\theta(o \mid q)
    \Bigr],
\end{equation}
where the advantage function is defined as
\begin{equation}
    A^{\pi_\theta}(q,o)
    =
    R(q,o)
    -
    \mathbb{E}_{o' \sim \pi_\theta(\cdot \mid q)}
    \bigl[ R(q,o') \bigr].
\end{equation}

In practice, the expectation over the policy is approximated using $N$ sampled
outputs $\{o_j\}_{j=1}^N$, yielding the empirical estimator
\begin{equation}
    \hat{A}^{\pi_\theta}(q,o_i)
    =
    R(q,o_i)
    -
    \frac{1}{N}
    \sum_{j=1}^N R(q,o_j),
    \quad
    o_j \sim \pi_\theta(\cdot | q).
\end{equation}
This group-wise baseline corresponds to the advantage estimator used in GRPO.

To enable multiple optimization steps on the same batch of data, we introduce the importance sampling ratio between the current policy $\pi_\theta$ and the behavior policy $\pi_{\theta_{\text{old}}}$:\begin{equation}r_\theta(q,o)=\frac{\pi_\theta(o \mid q)}{\pi_{\theta_{\text{old}}}(o \mid q)}.\end{equation}Consequently, the final surrogate objective with clipping—designed to constrain policy updates and ensure training stability—is formulated as:\begin{equation}\begin{aligned}
    &\mathcal{J}_{\text{CLIP}}(\theta)=\mathbb{E}_{q \sim \mathcal{D}, o \sim \pi_{\theta_{\text{old}}}(\cdot \mid q)}\\&\biggl[\min\Bigl(r_\theta(q,o) \hat{A}(q,o), \text{clip}\bigl(r_\theta(q,o), 1-\epsilon, 1+\epsilon\bigr) \hat{A}(q,o)\Bigr)\biggr],
\end{aligned}\end{equation}
where $\epsilon$ is a hyperparameter determining the clipping range, and $\hat{A}(q,o)$ is the group-wise advantage estimator defined previously.

\section{Method}
\label{sec:method}

\subsection{Thickening-to-Thinning Reward Design}
We introduce the \emph{Thickening-to-Thinning (T2T)} reward, a competence-aware reward shaping scheme for RLVR. The design augments the standard 0/1 verifier score with two ingredients: an on-policy pass-rate and a length score.

\paragraph{On-policy pass-rate.}
Following GRPO, for a fixed query $q$ we sample a group of $K$ outputs
$\{o_i\}_{i=1}^K \sim \pi_\theta(\cdot\mid q)$ and compute the on-policy pass-rate
\begin{equation}
\hat p_\theta(q)\;\triangleq\;\frac{1}{K}\sum_{i=1}^K \mathcal{V}(q,o_i).
\end{equation}

\paragraph{Length score.}
Let $\ell(o)$ denote the length of $o$ (e.g., number of tokens). We define a
normalized length score $s_L(o)\in[0,1]$ via clipping:
\begin{equation}
s_L(o)\;\triangleq\;\mathrm{clip}\!\left(\frac{\ell(o)-L_{\min}}{L_{\max}-L_{\min}},\,0,\,1\right).
\end{equation}
Here $L_{\min}$ and $L_{\max}$ are hyper-parameters specifying the effective range for the length score. The length scores of outputs shorter than $L_{\min}$ or longer than $L_{\max}$ are clipped to $0$ and $1$, respectively, so that the shaping signal does not grow unboundedly at extreme lengths.

\paragraph{Thickening-to-Thinning (T2T) reward.}
We define
\begin{equation}
R_{\mathrm{T2T}}(q,o,\theta)=
\begin{cases}
1-\alpha\, s_L(o)\,p_\theta(q), & \mathcal{V}(q,o)=1,\\[2pt]
\alpha\, s_L(o)\,(1-p_\theta(q)), & \mathcal{V}(q,o)=0.
\end{cases}
\qquad \alpha\in(0,\tfrac{1}{2})
\label{eq:R_T2T_piecewise}
\end{equation}
where $\alpha
$ is a hyper-parameter controlling the strength of the T2T shaping. In implementation, $p_\theta(q)$ is replaced by its on-policy estimate $\hat p_\theta(q)$:
\begin{equation}
\hat R_{\mathrm{T2T}}(q,o,\theta)=
\begin{cases}
1-\alpha\, s_L(o)\,\hat p_\theta(q), & \mathcal{V}(q,o)=1,\\[2pt]
\alpha\, s_L(o)\,(1-\hat p_\theta(q)), & \mathcal{V}(q,o)=0.
\end{cases}
\label{eq:R_T2T_piecewise_hat}
\end{equation}
We compute advantages and perform policy optimization following the same
procedure as GRPO. In implementation, $\hat p_\theta(q)$ is treated as a stop-gradient
statistic: although $\hat p_\theta(q)$ is computed from on-policy rollouts and thus
implicitly depends on $\theta$, we detach it before computing the policy
gradient, so that backpropagation only flows through $\log\pi_\theta(o\mid q)$.
This avoids high-variance gradient terms from the group-level estimator and
matches the standard practice in GRPO-style methods. Consequently, the
resulting update signal does not directly correspond to the gradient of any
fixed expected reward; we characterize the objective it actually optimizes in
Section~\ref{sec:analysis_t2t}.

\subsection{Analysis of T2T}
\label{sec:analysis_t2t}

\paragraph{Reward-induced ordering.}
For any fixed query $q$ and $\alpha\in(0,\tfrac12)$, since $s_L,p_\theta(q)\in[0,1]$, we have
\[
\inf_{\mathcal{V}=1}R_{\mathrm{T2T}}=1-\alpha\;>\;\alpha\;\ge\;\sup_{\mathcal{V}=0}R_{\mathrm{T2T}},
\]
so any correct output is rewarded above any incorrect one. Within each group, the shaping term breaks ties: among correct outputs, shorter is preferred ($-\alpha s_L p_\theta(q)$); among incorrect ones, longer is preferred ($+\alpha s_L(1-p_\theta(q))$). This yields the ordering
\[
\text{correct short}\;\succ\;\text{correct long}\;\succ\;\text{incorrect long}\;\succ\;\text{incorrect short},
\]
which prioritizes correctness, encourages longer attempts when the model struggles, and favors concise solutions once correctness is achieved. Unlike binary-reward GRPO, T2T provides non-trivial signals even when sampled outputs are all correct or all incorrect. Furthermore, the introduction of $p_\theta(q)$ ensures that for hard incorrect problems, T2T offers stronger encouragement for exploratory reasoning compared with easy incorrect cases; for simple solved problems, it imposes greater emphasis on concise reasoning than it does for correctly solved hard problems.

\paragraph{Gradient analysis.}
Recall that in implementation $\hat p_\theta(q)$ is detached from the computation graph.
With standard policy-gradient computation, the per-query update signal is then
\begin{equation}
  \frac{1}{K}\sum_{i=1}^K\hat R_{\mathrm{T2T}}(q,o_i,\theta)\, g_\theta(q,o_i),
\end{equation}
which is an estimator of
\begin{equation}
  \mathbb{E}_{o\sim \pi_\theta(\cdot\mid q)}\!\left[R_{\mathrm{T2T}}(q,o,\theta)\, g_\theta(q,o)\right],
  \label{eq:expected_grad_signal}
\end{equation}
where $g_\theta(q,o):=\nabla_\theta \log \pi_\theta(o\mid q)$.
Because $R_{\mathrm{T2T}}$ depends implicitly on $\theta$ through $p_\theta(q)$,
Eq.~\eqref{eq:expected_grad_signal} is \emph{not} equal to
$\nabla_\theta\mathbb{E}_{o\sim \pi_\theta}[R_{\mathrm{T2T}}]$, so the
expression alone does not directly reveal the underlying objective. To clarify
its effect, we provide a decomposition. Define the auxiliary reward
\begin{equation}
    \tilde R(q,o,\theta):=\mathcal{V}(q,o)+\alpha\, s_L(o)\bigl(1-p_\theta(q)\bigr),
\end{equation}
and its expected objective
\begin{equation}
    \tilde J_q(\theta)
    \;:=\;
    \mathbb{E}_{o\sim \pi_\theta(\cdot\mid q)}[\tilde R(q,o,\theta)]
    \;=\;
    p_\theta(q)+\alpha\,\mathbb{E}_{o\sim \pi_\theta(\cdot\mid q)}[s_L(o)]\bigl(1-p_\theta(q)\bigr).
\end{equation}
Notice that
\begin{equation}
R_{\mathrm{T2T}}(q,o,\theta)
\;\triangleq\;
\mathcal{V}(q,o)\bigl(1-\alpha\, s_L(o)\, p_\theta(q)\bigr)
\;+\;\bigl(1-\mathcal{V}(q,o)\bigr)\,\alpha\, s_L(o)\bigl(1- p_\theta(q)\bigr),
\label{eq:R_T2T_gate}
\end{equation}
thus, a direct calculation gives
$R_{\mathrm{T2T}}(q,o,\theta)=\tilde R(q,o,\theta)-\alpha\, s_L(o)\,\mathcal{V}(q,o)$.

\begin{theorem}[T2T update-signal decomposition]
\label{thm:decomposition}
For any fixed query $q$,
\begin{equation}
\mathbb{E}_{o\sim \pi_\theta(\cdot\mid q)}[R_{\mathrm{T2T}}\, g_\theta]
\;=\;
\nabla_\theta \tilde J_q(\theta)
\;+\;
\alpha\,\mathbb{E}_{o\sim \pi_\theta(\cdot\mid q)}\!\left[\bigl(\mathbb{E}[s_L(o)]-s_L(o)\bigr)\,\mathcal{V}(q,o)\, g_\theta(q,o)\right].
\end{equation}
\end{theorem}
\noindent The proof is given in Appendix~\ref{app:decomposition_proof}.

This decomposition admits a clean interpretation. Viewed as a function of
$p_\theta(q)$ with $\mathbb{E}[s_L(o)]$ held fixed, $\tilde J_q$ is monotone
increasing in $p_\theta(q)$, so its optimum is still attained at
$p_\theta(q)=1$; correctness remains the dominant target. The shaping term
$\alpha\,\mathbb{E}[s_L(o)]\bigl(1-p_\theta(q)\bigr)$ matters most when
$p_\theta(q)$ is small: in that regime, it supplies a non-trivial optimization
signal through length, whereas vanilla binary-reward GRPO often provides almost
no gradient when nearly all sampled outputs are incorrect. This is the
\textbf{thickening} effect.
The second term acts only on correct samples ($\mathcal{V}=1$). Its
coefficient $\mathbb{E}[s_L(o)]-s_L(o)$ is positive for shorter-than-average
correct responses and negative for longer-than-average ones. Hence, once the
model starts producing correct answers, T2T adds an explicit update signal that
favors shorter correct trajectories. This is the \textbf{thinning} effect.

\section{Experiment}
\label{sec:experiment}
\subsection{Experimental Setup}
\paragraph{Models.}
To investigate the scalability and effectiveness of our proposed method across different model capacities, we conduct experiments on the Qwen and DeepSeek series, which are widely recognized for their mathematical reasoning capabilities. Specifically, we utilize \texttt{DeepSeek-R1-Distill-Qwen-1.5B},
\texttt{Llama-3.2-3B-Instruct}, \texttt{Qwen2.5-3B}, \texttt{Qwen3-4B}, and \texttt{Qwen3-14B} as our base policies. For the Qwen3 series, we utilize the \textbf{non-thinking mode}. These models serve as the initialization for all reinforcement learning experiments.

\paragraph{Training Setup.}
We construct our training set using \texttt{DAPO-MATH-17k}~\cite{dapo-math-17k}, a high-quality dataset specifically filtered for mathematical reasoning. We train for a total of 5 epochs. All models are trained using the \textbf{Verl} framework~\cite{verl}. The maximum new tokens is set to 8192 and the group size is set to 8. Detailed hyperparameter configurations are provided in Appendix~\ref{app:hyperparameters}.

\paragraph{Baseline Settings.}
We compare \textbf{T2T} against standard GRPO and three advanced baselines that explicitly differentiate between positive and negative samples to refine the learning signal. Detailed implementation specifics and hyperparameter configurations are provided in Appendix~\ref{app:baseline_settings}. The specific methods are as follows:
\begin{itemize}
    \item \textbf{GRPO}~\cite{DeepSeekR1, DeepSeekMath}: The vanilla Group Relative Policy Optimization algorithm without additional length or gradient modifications.
    \item \textbf{LASER}~\cite{LASER}: A reward shaping method that encourages efficient reasoning by augmenting the reward for positive samples with a length penalty term $\alpha \cdot \mathbb{I}(L < L_{\text{target}})$, while keeping rewards for negative samples unchanged. 
    \item \textbf{W-REINFORCE}~\cite{NSR}: A gradient reweighting approach that retains the full gradient for negative samples to facilitate unlearning, while scaling down the gradient for positive samples by a factor $\rho$.
    \item \textbf{EntroPIC}~\cite{EntroPIC}: An entropy-stabilization method that utilizes Proportional-Integral (PI) control to dynamically adjust the loss weights of positive and negative samples based on the discrepancy between the current policy entropy and a target entropy.
\end{itemize}
We note that due to computational constraints, for \texttt{Llama-3.2-3B-Instruct} and \texttt{Qwen3-14B}, only the standard GRPO baseline is implemented.

\paragraph{Evaluation Setup.}
We evaluate the models on four widely used and challenging mathematical reasoning benchmarks: \textbf{MATH-500}~\cite{math-dataset,lightman2023let}, \textbf{AIME'24}, \textbf{AIME'25}, and \textbf{AMC'23}. 
During evaluation, we sample $n=64$ responses per prompt for MATH-500, and $n=128$ for AIME'24, AIME'25, and AMC'23. The sampling temperature is set to 0.7 with a top-$p$ of 0.95.

We report the \textbf{Pass@$k$} metric ($k \in \{1, 2, 4, 8, 16, 32, 64\}$). 
Specifically, \textbf{Pass@1} is calculated as the mean accuracy over the $n$ generated samples. 
For $k > 1$, to avoid high variance from limited sampling, we employ the unbiased estimator proposed by \citet{chen2021evaluating}. Given $n$ generated samples with $c$ correct responses, Pass@$k$ is calculated as:
\begin{equation}
    \text{Pass}@k = \mathbb{E}\left[ 1 - \frac{\binom{n-c}{k}}{\binom{n}{k}} \right].
\end{equation}

\subsection{Main Results on Mathematical Benchmarks}

\begin{table*}[t]
  \caption{\textbf{Main Results} on mathematical reasoning benchmarks. We report \textbf{Pass@1} and \textbf{Pass@64} performance using the unbiased estimator. The best results are highlighted in \textbf{bold}, and the second best results are \underline{underlined}.}
  \label{tab:main_results}
  \centering
  \resizebox{1.0\textwidth}{!}{
  \begin{tabular}{c|l|cc|cc|cc|cc}
    \toprule
    \multicolumn{1}{c}{\multirow{2}{*}{Base Model}} &
    \multicolumn{1}{c}{\multirow{2}{*}{Method}} &
    \multicolumn{2}{c}{\textbf{MATH-500}} &
    \multicolumn{2}{c}{\textbf{AIME'24}} &
    \multicolumn{2}{c}{\textbf{AIME'25}} &
    \multicolumn{2}{c}{\textbf{AMC'23}} \\
    \cmidrule(lr){3-4} \cmidrule(lr){5-6} \cmidrule(lr){7-8} \cmidrule(lr){9-10}
    \multicolumn{2}{c|}{} & Pass@1 & Pass@64 & Pass@1 & Pass@64 & Pass@1 & Pass@64 & Pass@1 & Pass@64 \\
    \toprule

    \multirow{6}{*}{\rotatebox{0}{DeepSeek-R1-Distill-Qwen-1.5B}} 
    & Initial Model & 66.0 & 85.0 & 20.5 & 46.9 & 14.8 & 28.2 & 50.9 & 86.4 \\
    & GRPO          & \textbf{68.9} & 84.0 & 20.6 & 47.7 & 14.8 & \textbf{28.9} & 57.3 & \textbf{92.7} \\
    & LASER         & 66.5 & 85.4 & 20.8 & 48.3 & 14.9 & 28.4 & 57.8 & 91.5 \\
    & W-REINFORCE   & \underline{67.5} & \underline{85.8} & 20.4 & 46.8 & 14.7 & 28.0 & 58.2 & 91.6 \\
    & EntroPIC      & 66.2 & 84.8 & \underline{21.0} & \underline{49.5} & \underline{15.0} & 28.5 & \underline{58.5} & 91.4 \\
    \cmidrule{2-10}
    & \textbf{T2T (Ours)} & 66.9 & \textbf{86.1} & \textbf{21.3} & \textbf{50.7} & \textbf{15.2} & \underline{28.8} & \textbf{59.0} & \underline{91.7} \\
    \midrule

    \multirow{3}{*}{\rotatebox{0}{Llama-3.2-3B-Instruct}} 
    & Initial Model & 24.9 & 32.5 & 12.6 & 13.3 & 0.0 & 0.0 & 15.4 & 22.5 \\
    & GRPO          & \underline{34.3} & \underline{45.0} & \underline{14.7} & \underline{20.9} & \underline{3.2} & \underline{5.3} & \underline{29.2} & \underline{43.7} \\
    \cmidrule{2-10}
    & \textbf{T2T (Ours)} & \textbf{36.2} & \textbf{47.6} & \textbf{15.5} & \textbf{25.4} & \textbf{3.6} & \textbf{6.6} & \textbf{32.0} & \textbf{49.1} \\
    \midrule

    \multirow{6}{*}{\rotatebox{0}{Qwen2.5-3B}} 
    & Initial Model & 40.6 & 49.1 & 7.1 & 14.7 & 3.4 & 4.6 & 43.9 & 59.0 \\
    & GRPO & 47.4 & 59.9 & 4.8 & 14.7 & \textbf{7.7} & 15.6 & 49.1 & 63.4 \\
    & LASER     & \underline{48.4} & 57.9 & 6.0 & 12.2 & 4.6 & 12.2 & 51.9 & 66.2 \\
    & W-REINFORCE    & 46.5 & 60.2 & \textbf{12.6} & \textbf{22.8} & \underline{5.8} & \underline{15.9} & 52.5 & 70.9 \\
    & EntroPIC & 44.6 & \underline{61.3} & 7.2 & 18.3 & 3.6 & 9.2 & \underline{54.7} & \underline{71.9} \\
    \cmidrule{2-10}
    & \textbf{T2T (Ours)} & \textbf{49.8} & \textbf{64.6} & \underline{9.6} & \underline{21.7} & 5.4 & \textbf{16.7} & \textbf{57.0} & \textbf{77.0}  \\
    \midrule

    \multirow{6}{*}{\rotatebox{0}{Qwen3-4B}} 
    & Initial Model & 80.2 & \textbf{91.2} & 22.2 & 37.0 & 14.7 & 35.4 & 71.4 & 92.4 \\
    & GRPO & \textbf{81.2} & 90.5 & \underline{29.3} & 48.6 & \underline{21.3} & 48.5 & 75.8 & 91.2 \\
    & LASER     & 80.6 & 90.7 & 28.2 & 45.4 & \underline{21.3} & 46.9 & \underline{76.1} & 91.1 \\
    & W-REINFORCE   & \underline{80.9} & 90.6 & 28.5 & \underline{50.7} & 20.8 & \textbf{50.9} & 75.8 & \underline{93.9} \\
    & EntroPIC & 80.5 & 90.8 & 28.8 & 47.9 & 21.0 & 48.0 & 75.9 & 91.5 \\
    \cmidrule{2-10}
    & \textbf{T2T (Ours)} & \textbf{81.2} & \underline{91.1} & \textbf{33.4} & \textbf{56.3} & \textbf{23.0} & \underline{49.6} & \textbf{78.0} & \textbf{95.3} \\
    \midrule
    \multirow{3}{*}{\rotatebox{0}{Qwen3-14B}} 
    & Initial Model       & 82.2 & 87.2 & 27.1 & 48.3 & 19.6 & 47.3 & 74.6 & 91.2  \\
    & GRPO                & \underline{83.3} & \underline{88.8} & \underline{29.1} & \underline{56.1} & \underline{22.4} & \underline{51.5} & \underline{80.7} & \underline{94.6} \\
    \cmidrule{2-10}
    & \textbf{T2T (Ours)} & \textbf{85.1} & \textbf{92.7} & \textbf{39.6} & \textbf{59.6} & \textbf{24.5} & \textbf{57.9} & \textbf{81.4} & \textbf{98.4}  \\
    \bottomrule
  \end{tabular}
  }
\end{table*}


The main results across four mathematical reasoning benchmarks are summarized in Table~\ref{tab:main_results}. The full experimental results are provided in Appendix~\ref{app:full_results}. Our experiments yield several key observations regarding the effectiveness and scalability of the proposed method:

\paragraph{Competitive Performance and Robust Improvements.}
As shown in Table~\ref{tab:main_results}, \textbf{T2T} demonstrates highly competitive performance across most tested model capacities.
Notably, on the Qwen2.5-3B and Qwen3-4B models, our method yields consistent gains, particularly excelling in the AMC'23 and AIME benchmarks compared to the standard GRPO baseline and other variants like LASER and W-REINFORCE. 
On the largest evaluated model, Qwen3-14B, T2T exhibits dominant performance, securing the top spot in all Pass@1 and Pass@64 metrics, which validates the method's strong scalability.

\paragraph{Effectiveness of RL Training.}
Comparing the results with the \textit{Initial Model}, it is evident that Reinforcement Learning significantly boosts mathematical reasoning capabilities. All RL-based methods show a substantial jump in performance compared to the base model. However, T2T generally maximizes this gain on models with sufficient capacity (3B+), suggesting that our approach utilizes the training data more efficiently by better distinguishing between high-quality and low-quality reasoning paths. Detailed analysis of T2T training dynamics are provided in Appendix~\ref{app:analysis}.


\subsection{Controlled Analysis of Reasoning Trajectory Length}

To isolate the effect of our T2T method on reasoning trajectory length, we conduct a controlled pairwise comparison between T2T and vanilla GRPO. Specifically, we restrict the analysis to samples that are solved correctly by both methods, eliminating confounding factors from differences in problem difficulty or correctness. We report the average reasoning length (in tokens) across five model families and four mathematical reasoning benchmarks, as shown in Table ~\ref{tab:control_length}.
\begin{table}[t]
\centering
\caption{Controlled analysis of reasoning trajectory lengths (token counts) on samples solved correctly by both GRPO and T2T.}
\label{tab:control_length}
\resizebox{0.8\linewidth}{!}{ 
\begin{tabular}{llcccc}
\toprule
Model & Method & MATH500 & AIME24 & AIME25 & AMC23 \\
\midrule
\multirow{2}{*}{DS-1.5B} & GRPO & 1453 & 3124 & 3510 & 2726 \\
& T2T  & 1417 & 2786 & 3348 & 2814 \\
\midrule
\multirow{2}{*}{Llama3.2-3B} & GRPO & 367 & 797 & 1483 & 651 \\
& T2T  & 338 & 653 & 1237 & 584 \\
\midrule
\multirow{2}{*}{Qwen2.5-3B} & GRPO & 1534 & 607 & 1521 & 1173 \\
& T2T  & 1416 & 574 & 1392 & 1068 \\
\midrule
\multirow{2}{*}{Qwen3-4B} & GRPO & 664 & 1718 & 1443 & 1732 \\
& T2T  & 612 & 1543 & 1316 & 1577 \\
\midrule
\multirow{2}{*}{Qwen3-14B} & GRPO & 687 & 1804 & 1496 & 1798 \\
& T2T  & 534 & 1327 & 1384 & 1435 \\
\bottomrule
\end{tabular}
}
\end{table}

The results reveal a clear and consistent pattern: across nearly all model scales and benchmarks, T2T produces shorter average reasoning trajectories than GRPO on samples correctly solved by both methods. For instance, on the challenging AIME24 benchmark, T2T reduces the average reasoning length of Qwen3-14B from 1804 tokens to 1327 tokens, and that of Llama3.2-3B from 797 tokens to 653 tokens. This controlled comparison directly validates the thinning phase of our T2T framework: while standard GRPO assigns uniform rewards to all correct outputs and lacks incentives for conciseness, T2T dynamically encourages compact reasoning for mastered problems. The model accordingly prunes redundant exploratory steps and unnecessary derivations to converge on efficient solution paths, which aligns with the human learning progression from verbose exploration to condensed knowledge consolidation. More analysis of trajectory length are provided in Appendix ~\ref{app:test_set_length} and Appendix ~\ref{app:length_dynamics}.

\subsection{Ablation Study}
\label{sec:ablation}

\begin{table*}[t]
  \caption{\textbf{Ablation Study on Qwen2.5-3B.} We investigate the impact of Difficulty Awareness, Thickening (exploration reward), and Thinning (efficiency penalty). \textbf{Pass@1} and \textbf{Pass@64} are reported. The best results are highlighted in \textbf{bold}.}
  \label{tab:ablation}
  \centering
  \resizebox{1.0\textwidth}{!}{
  \begin{tabular}{l|cc|cc|cc|cc}
    \toprule
    \multicolumn{1}{c|}{\multirow{2}{*}{\textbf{Method Variant}}} &
    \multicolumn{2}{c|}{\textbf{MATH-500}} &
    \multicolumn{2}{c|}{\textbf{AIME'24}} &
    \multicolumn{2}{c|}{\textbf{AIME'25}} &
    \multicolumn{2}{c}{\textbf{AMC'23}} \\
    \cmidrule(lr){2-3} \cmidrule(lr){4-5} \cmidrule(lr){6-7} \cmidrule(lr){8-9}
    & Pass@1 & Pass@64 & Pass@1 & Pass@64 & Pass@1 & Pass@64 & Pass@1 & Pass@64 \\
    \midrule
    
    w/o Difficulty Awareness & 49.5 & 64.4 & 8.9 & 20.5 & 4.9 & 15.8 & 56.2 & 75.8 \\
    
    w/o Thinning (Thickening Only) & 49.1 & 63.3 & 8.5 & 20.8 & 4.6 & 16.0 & 55.1 & 75.2 \\
    
    w/o Thickening (Thinning Only) & 48.3 & 60.1 & 8.2 & 18.5 & 4.2 & 13.9 & 55.8 & 72.5 \\
    
    \midrule
    
    \textbf{T2T (Full)} & \textbf{49.8} & \textbf{64.6} & \textbf{9.6} & \textbf{21.7} & \textbf{5.4} & \textbf{16.7} & \textbf{57.0} & \textbf{77.0} \\
    \bottomrule
  \end{tabular}
  }
\end{table*}

To explicitly evaluate the contribution of each component in our proposed method, we conduct an ablation study using the \texttt{Qwen2.5-3B} model. We compare the full \textbf{T2T} against three variants:
\begin{itemize}
    \item \textbf{w/o Difficulty Awareness:} We remove the group pass rate term from the reward formulation, applying a static length reward/penalty regardless of problem difficulty.
    \begin{equation}
    R_{\text{w/o Difficulty Awareness}}(q,o)=
    \begin{cases}
    1-\alpha\, s_L(o)\,, & \mathcal{V}(q,o)=1,\\[2pt]
    \alpha\, s_L(o)\,, & \mathcal{V}(q,o)=0.
    \end{cases}
    \label{eq:R_T2T_piecewise_w/o_diff_aware}
    \end{equation}
    \item \textbf{w/o Thinning (Thickening Only):} We retain difficulty awareness but only apply the positive reward for incorrect responses (encouraging exploration on hard tasks). The penalty for correct responses is removed.
    \item \textbf{w/o Thickening (Thinning Only):} We retain difficulty awareness but only apply the length penalty for correct responses (encouraging efficiency on easy tasks). The reward for incorrect responses is removed.
    \item \textbf{T2T (Full):} The complete method incorporating difficulty-aware Thickening and Thinning.
\end{itemize}

The results are summarized in Table~\ref{tab:ablation}.

    
    
    
    
    

\paragraph{Impact of Difficulty Awareness.}
Comparing the full model with the \textit{w/o Difficulty Awareness} variant, we observe a clear performance drop when the adaptive pass-rate scaling is removed. This confirms that treating all samples equally is suboptimal. A static length bias fails to distinguish between the need for exploration in hard problems and the need for conciseness in solved problems, leading to conflicting learning signals.

\paragraph{Impact of Thickening and Thinning Mechanics.}
As shown in the ablation results, removing either the ``Thickening'' or ``Thinning'' component leads to a general degradation in performance across most metrics, validating that both phases are essential for the full effectiveness of T2T. 
The \textit{w/o Thinning} variant (Thickening Only) retains the incentive for expansive exploration on difficult or incorrect problems. However, its performance drops indicate that without the subsequent ``Thinning'' phase—which encourages abstraction and compression on correct solutions—the model fails to refine its verbose reasoning into efficient knowledge, mirroring the human learning process where understanding must be condensed after expansion.
Conversely, the \textit{w/o Thickening} variant (Thinning Only) applies penalties for redundancy but lacks the reward mechanism to encourage longer chains when the model fails. The significant decline in performance here suggests that skipping the ``Thickening'' phase prevents the model from engaging in the deep exploration necessary to solve hard problems, effectively trying to ``thin'' the knowledge before it has been sufficiently ``thickened.'' 
These findings confirm our hypothesis based on Hua Luogeng's principle: mathematical reasoning capabilities are best developed by combining expansive exploration (Thickening) with disciplined abstraction (Thinning), rather than relying on either mechanism in isolation.

\section{Discussion}
\label{sec:discussion}

\paragraph{Positioning T2T within the RLVR algorithmic stack.}

A useful way to situate T2T is to view RLVR algorithms as comprising three loosely coupled layers: the \emph{reward} layer, which specifies $R(q,o)$; the \emph{advantage} layer, which estimates $\hat{A}$ from sampled rewards (e.g., GRPO's group-relative baseline, Dr.\,GRPO's debiased estimator, or a learned value critic); and the \emph{update} layer, which converts advantages into policy updates (e.g., PPO-style clipping, KL regularization, or W-REINFORCE-style reweighting). T2T modifies only the reward layer. The competence statistic $\hat{p}_\theta(q)$ is computed from on-policy rollouts and detached from the computation graph (Section~\ref{sec:method}), so the resulting reward enters the optimization pipeline as an ordinary scalar signal and does not interact with importance sampling, clipping, or KL terms. As a consequence, T2T is \emph{orthogonal} to RLVR variants that target the advantage or update layers: it can in principle be combined with frameworks like Dr.\,GRPO~\cite{DrGRPO}, GSPO~\cite{zheng2025group}, DAPO~\cite{dapo-math-17k}, or off-policy extensions such as Luffy~\cite{Luffy}, without algorithmic conflict and without additional computational overhead.

\paragraph{Beyond RLVR: thickening-to-thinning as an implicit curriculum.}
More broadly, the thickening-to-thinning pattern can be viewed as an implicit, reward-driven curriculum: rather than scheduling task difficulty externally or annealing exploration coefficients by hand, the model's own competence on each query determines how much exploratory budget it receives. We see this as a small step toward training procedures whose exploration--exploitation trade-off adapts continuously to the learner's evolving state, and a direction worth investigating in settings beyond verifiable mathematical reasoning.

\section{Conclusion}
This work introduces Thickening-to-Thinning (T2T), a competence and length conditioned reward shaping strategy that explicitly distinguishes between exploratory reasoning and efficient solution refinement within RLVR.
Importantly, T2T is lightweight to implement, requires no additional models or inference-time overhead, and can be readily combined with algorithmic extensions that modify optimization or sampling procedures rather than reward semantics.
Extensive experiments demonstrate that T2T not only improves final performance and generalization across diverse benchmarks, but also induces healthier training dynamics, mitigating common RLVR pathologies such as brittle exploration or excessive verbosity.
More broadly, our results suggest that mimicking key aspects of human learning—progressing from expansive exploration to deliberate compression—may offer a fruitful direction for advancing reasoning in large language models beyond static reward formulations.

\newpage
\bibliography{neurips_2026}
\bibliographystyle{plainnat}


\appendix

\section{Proof of Theorem 4.1}
\label{app:decomposition_proof}

For brevity, throughout this appendix we write
\[
\mathbb{E}[\cdot]\;:=\;\mathbb{E}_{o\sim \pi_\theta(\cdot\mid q)}[\cdot],
\qquad
g_\theta(q,o)\;:=\;\nabla_\theta\log\pi_\theta(o\mid q).
\]
Recall that in implementation $\hat p$ is detached from the computation graph,
so the policy gradient flows only through $\log\pi_\theta(o\mid q)$. The
quantity we analyze, $\mathbb{E}[R_{\mathrm{T2T}}\, g_\theta]$, is the
population-level counterpart of this update.

\begin{proof}
The proof proceeds in three steps: (i) expand $\nabla_\theta \tilde J_q$ via
the score-function identity, (ii) evaluate $\nabla_\theta p_\theta(q)$, and
(iii) compare with $\mathbb{E}[R_{\mathrm{T2T}}\, g_\theta]$.

\smallskip
\noindent\textbf{Step 1: Score-function expansion of $\nabla_\theta \tilde J_q$.}
Since $\tilde R(q,o,\theta)$ depends on $\theta$ both through the sampling
distribution $\pi_\theta(\cdot\mid q)$ and explicitly through $p_\theta(q)$,
differentiating $\tilde J_q(\theta) = \mathbb{E}[\tilde R]$ yields
\begin{equation}
\nabla_\theta \tilde J_q(\theta)
\;=\;
\underbrace{\mathbb{E}\!\left[\nabla_\theta \tilde R\right]}_{\text{explicit dependence}}
\;+\;
\underbrace{\mathbb{E}\!\left[\tilde R\, g_\theta\right]}_{\text{score-function term}}.
\label{eq:proof_step1}
\end{equation}
Among the components of $\tilde R = \mathcal{V}(q,o) + \alpha\, s_L(o)\,(1-p_\theta(q))$, only $p_\theta(q)$ depends explicitly on $\theta$, so
\begin{equation}
\nabla_\theta \tilde R(q,o,\theta)
\;=\;
-\alpha\, s_L(o)\,\nabla_\theta p_\theta(q).
\label{eq:proof_explicit_grad}
\end{equation}

\smallskip
\noindent\textbf{Step 2: Evaluating $\nabla_\theta p_\theta(q)$.}
Applying the score-function identity to
$p_\theta(q) = \mathbb{E}[\mathcal{V}(q,o)]$,
\begin{equation}
\nabla_\theta p_\theta(q)
\;=\;
\mathbb{E}\!\left[\mathcal{V}(q,o)\, g_\theta(q,o)\right].
\label{eq:proof_grad_p}
\end{equation}
Substituting Eqs.~\eqref{eq:proof_explicit_grad}--\eqref{eq:proof_grad_p} into
Eq.~\eqref{eq:proof_step1} and using the fact that $\mathbb{E}[s_L(o)]$ is a
constant with respect to $o$,
\begin{align}
\nabla_\theta \tilde J_q(\theta)
&\;=\;
\mathbb{E}[\tilde R\, g_\theta]
\;-\;
\alpha\, \mathbb{E}[s_L(o)]\,\mathbb{E}[\mathcal{V}(q,o)\, g_\theta]
\nonumber\\
&\;=\;
\mathbb{E}\!\left[\bigl(\tilde R-\alpha\,\mathbb{E}[s_L(o)]\,\mathcal{V}(q,o)\bigr)\, g_\theta\right].
\label{eq:proof_Jtilde}
\end{align}

\smallskip
\noindent\textbf{Step 3: Comparing with $\mathbb{E}[R_{\mathrm{T2T}}\, g_\theta]$.}
By the identity $R_{\mathrm{T2T}} = \tilde R - \alpha\, s_L(o)\,\mathcal{V}(q,o)$ established in the main text,
\begin{equation}
\mathbb{E}[R_{\mathrm{T2T}}\, g_\theta]
\;=\;
\mathbb{E}\!\left[\bigl(\tilde R-\alpha\, s_L(o)\,\mathcal{V}(q,o)\bigr)\, g_\theta\right].
\label{eq:proof_RT2T}
\end{equation}
Subtracting Eq.~\eqref{eq:proof_Jtilde} from Eq.~\eqref{eq:proof_RT2T},
\[
\mathbb{E}[R_{\mathrm{T2T}}\, g_\theta]
\;-\;
\nabla_\theta \tilde J_q(\theta)
\;=\;
\alpha\,\mathbb{E}\!\left[\bigl(\mathbb{E}[s_L(o)]-s_L(o)\bigr)\,\mathcal{V}(q,o)\, g_\theta\right],
\]
which rearranges to
\[
\mathbb{E}[R_{\mathrm{T2T}}\, g_\theta]
\;=\;
\nabla_\theta \tilde J_q(\theta)
\;+\;
\alpha\,\mathbb{E}\!\left[\bigl(\mathbb{E}[s_L(o)]-s_L(o)\bigr)\,\mathcal{V}(q,o)\, g_\theta(q,o)\right].
\]
\end{proof}

\section{Implementation Details}
\label{app:implementation_details}
\subsection{Hyperparameters}
\label{app:hyperparameters}

We provide the detailed hyperparameter configurations used in our experiments in Table~\ref{tab:hyperparameters}. All models are trained using the \textbf{Verl}~\cite{verl} framework with the settings specified below.

\begin{table}[h]
    \centering
    \caption{Hyperparameter settings for T2T.}
    \label{tab:hyperparameters}
    \vspace{5pt}
    \begin{tabular}{l|c}
        \toprule
        \textbf{Hyperparameter} & \textbf{Value} \\
        \midrule
        Learning Rate & 1e-6 \\
        Group Size & 8 \\
        Train Batch Size & 256 \\
        PPO Mini-Batch Size & 64 \\
        Max Response Length & 8192 \\
        Max Prompt Length & 1024 \\
        Rollout Temperature & 1.0 \\
        Evaluation Temperature & 0.7 \\
        Evaluation Top-$p$ & 0.95 \\
        Clip Ratio High & 0.2 \\
        Clip Ratio Low & 0.2 \\
        Total Training Steps & 275 \\
        \midrule
        \multicolumn{2}{l}{\textit{T2T Reward Parameters}} \\
        Scaling Factor ($\alpha$) & 0.2 \\
        Length Norm Max ($L_{\max}$) & 4096 \\
        Length Norm Min ($L_{\min}$) & 0 \\
        \bottomrule
    \end{tabular}
\end{table}

\subsection{Baseline Implementation Details}
\label{app:baseline_settings}

To ensure a fair comparison and reproducibility, we strictly follow the hyperparameter recommendations provided in the respective original publications for all baseline methods. The specific configurations are detailed below:

\paragraph{GRPO.}
We utilize the standard implementation of Group Relative Policy Optimization~\cite{DeepSeekR1} without any additional reward shaping or gradient modification terms. 

\paragraph{LASER.}
Following \citet{LASER}, we configure the length-penalty reward shaping to encourage conciseness in correct responses. We set the penalty coefficient and the target length threshold as follows:
\begin{itemize}
    \item \textbf{Penalty Coefficient ($\alpha$):} $0.2$
    \item \textbf{Target Length ($L_{\text{target}}$):} $4096$
\end{itemize}
This target length aligns with the maximum response length used in our main experiments.

\paragraph{W-REINFORCE.}
We adopt the \textbf{W-REINFORCE} baseline as recommended by \citet{NSR}. This method scales down the gradients of positive samples while retaining full gradients for negative samples to encourage exploration.
\begin{itemize}
    \item \textbf{Positive Sample Weight ($\rho$):} $0.1$
\end{itemize}

\paragraph{EntroPIC.}
For the entropy-stabilization method EntroPIC~\cite{EntroPIC}, we employ the Proportional-Integral (PI) control mechanism to dynamically adjust loss weights. The specific control parameters are set as:
\begin{itemize}
    \item \textbf{Target Entropy ($H_{\text{target}}$):} $0.1$
    \item \textbf{Proportional Gain ($K_p$):} $1.0$
    \item \textbf{Integral Gain ($K_i$):} $0.01$
\end{itemize}
These values correspond to the configuration for standard mathematical reasoning tasks as described in the original paper.

\section{Generalization Capabilities}
\label{app:generalization}

A critical concern in domain-specific reinforcement learning is the potential risk of overfitting to the target domain (mathematics), which may degrade the model's performance on other tasks—a phenomenon often referred to as the ``alignment tax.'' To evaluate the robustness and transferability of T2T, we conduct zero-shot evaluations on two out-of-domain benchmarks: \textbf{BBH} (Big Bench Hard)~\cite{suzgun2023challenging} for general logical reasoning and \textbf{HumanEval}~\cite{chen2021evaluating} for code generation.

We evaluate the generalization capabilities of the \textit{Initial Model}, standard GRPO, and other baseline methods (LASER, W-REINFORCE, EntroPIC) on Qwen2.5-3B and Qwen3-4B. The results are summarized in Table~\ref{tab:generalization}.

\begin{table*}[h]
  \caption{\textbf{Generalization Performance on Out-of-Domain Tasks.} We report the \textbf{average 3-shot Exact Match (EM)} across all 27 subtasks of BBH (Reasoning) and \textbf{Pass@k} ($k \in \{1, 5, 10\}$) for HumanEval (Coding). \textbf{Detailed breakdowns of performance on each BBH subtask are provided in Appendix~\ref{app:bbh_full_results}.}}
  \label{tab:generalization}
  \centering
  \resizebox{0.85\textwidth}{!}{
  \begin{tabular}{c|l|c|ccc}
    \toprule
    \multicolumn{1}{c|}{\multirow{2}{*}{\textbf{Base Model}}} & 
    \multicolumn{1}{c|}{\multirow{2}{*}{\textbf{Method}}} & 
    \textbf{BBH (Reasoning)} & 
    \multicolumn{3}{c}{\textbf{HumanEval (Coding)}} \\
    \cmidrule(lr){3-3} \cmidrule(lr){4-6}
    \multicolumn{1}{c|}{} & \multicolumn{1}{c|}{} & 3-shot EM & Pass@1 & Pass@5 & Pass@10 \\
    \midrule
    
    \multirow{6}{*}{Qwen2.5-3B} 
    & Initial Model          & 55.3 & 61.5 & 78.2 & 84.5 \\
    & GRPO & 55.9 & 64.2 & 80.5 & 86.0 \\
    & LASER    & 56.0 & 64.5 & 80.8 & 86.4 \\
    & W-REINFORCE        & 55.9 & 64.1 & 80.4 & 86.1 \\
    & EntroPIC & 55.7 & 63.8 & 80.0 & 85.8 \\
    \cmidrule{2-6}
    & \textbf{T2T (Ours)}    & 55.8 & 64.2 & 80.5 & 86.2 \\
    \midrule
    
    \multirow{6}{*}{Qwen3-4B} 
    & Initial Model          & 79.3 & 72.0 & 86.5 & 91.0 \\
    & GRPO & 80.8 & 75.2 & 88.5 & 92.8 \\
    & LASER    & 80.7 & 75.5 & 88.8 & 93.0 \\
    & W-REINFORCE          & 80.7 & 75.1 & 88.4 & 92.6 \\
    & EntroPIC & 80.0 & 74.8 & 88.0 & 92.2 \\
    \cmidrule{2-6}
    & \textbf{T2T (Ours)}    & 80.6 & 75.7 & 89.0 & 93.5 \\
    \midrule
  \end{tabular}
  }
\end{table*}

\paragraph{Generalization to Broader Reasoning Tasks.}
We observe that reinforcement learning demonstrates robust generalization capabilities beyond mathematical domains. As shown in Table~\ref{tab:generalization}, all evaluated RL methods maintain or slightly improve upon the Initial Model's performance on out-of-domain benchmarks, such as BBH (general reasoning) and HumanEval (coding). This suggests that the reasoning patterns acquired during mathematical training effectively transfer to broader logical and algorithmic tasks. T2T aligns with this trend, confirming that our method enhances domain-specific performance while preserving general reasoning abilities.

\section{Analysis of Training Dynamics}
\label{app:analysis}

To investigate the learning behavior of T2T, we visualize the evolution of Training Accuracy and Policy Entropy during the fine-tuning process. Figure~\ref{fig:acc_dynamics} and figure~\ref{fig:entropy_dynamics} presents the comparative results on \texttt{Qwen2.5-3B} and \texttt{Qwen3-4B}.

\begin{figure*}[t]
  \centering
  \begin{minipage}{0.49\textwidth}
    \centering
    \includegraphics[width=\linewidth]{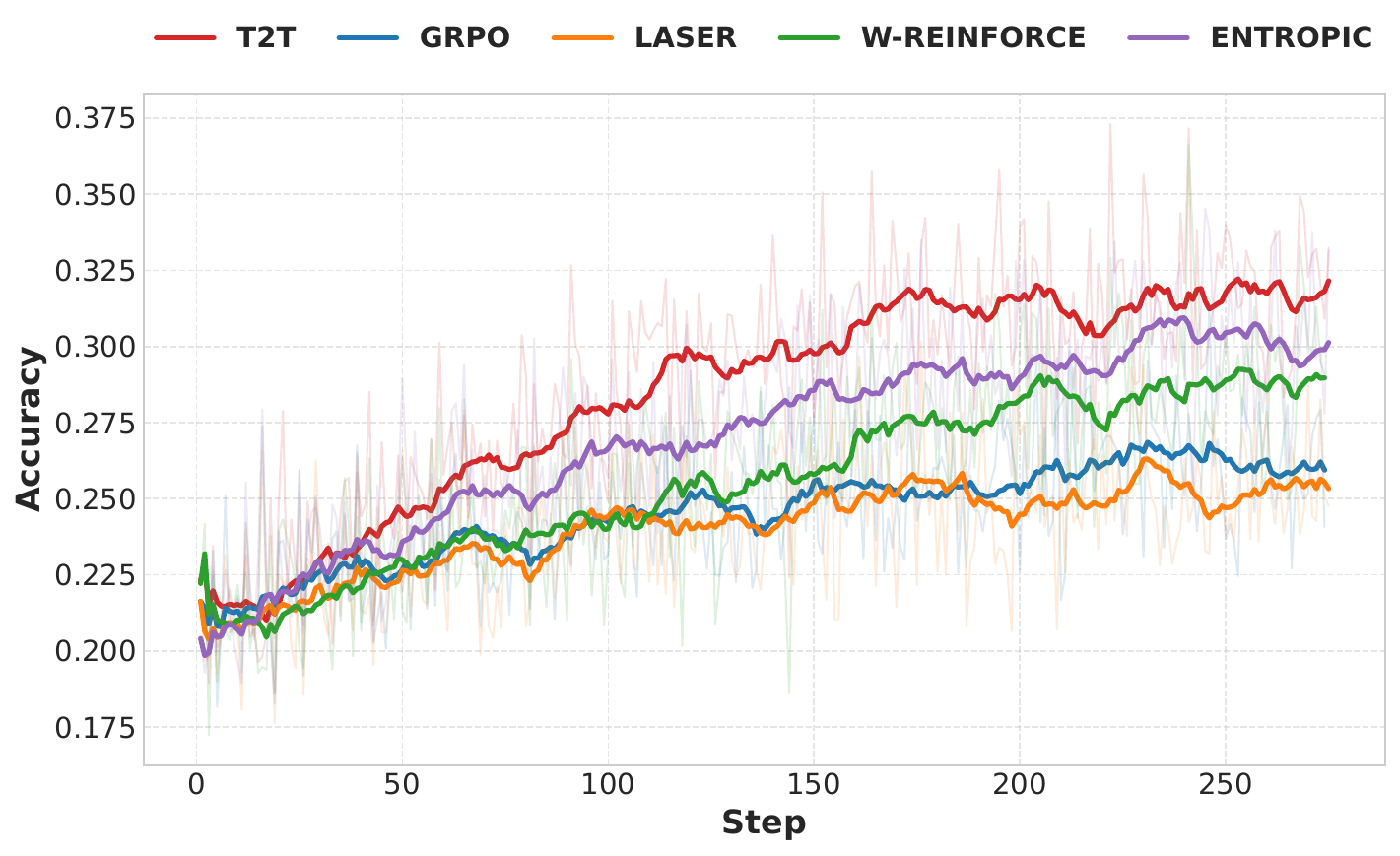} 
  \end{minipage}
  \hfill
  \begin{minipage}{0.49\textwidth}
    \centering
    \includegraphics[width=\linewidth]{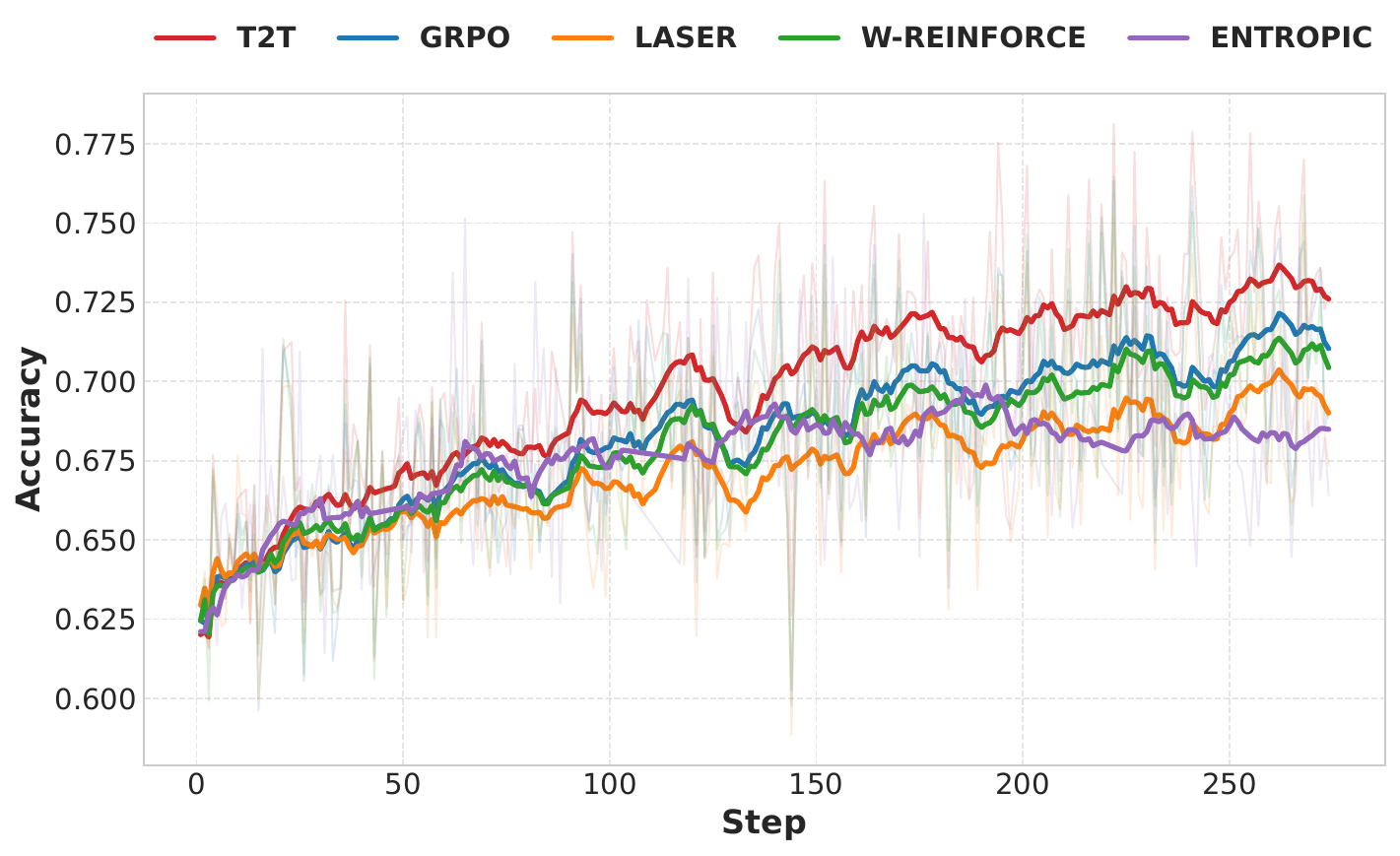} 
  \end{minipage}

  \caption{\textbf{Training Accuracy Evolution.} The plot on the left corresponds to \texttt{Qwen2.5-3B}, and the plot on the right corresponds to \texttt{Qwen3-4B}. Across both model scales, T2T demonstrates superior learning efficiency compared to the baseline.}
  \label{fig:acc_dynamics}
\end{figure*}

\begin{figure*}[t]
  \centering
  \begin{minipage}{0.49\textwidth}
    \centering
    \includegraphics[width=\linewidth]{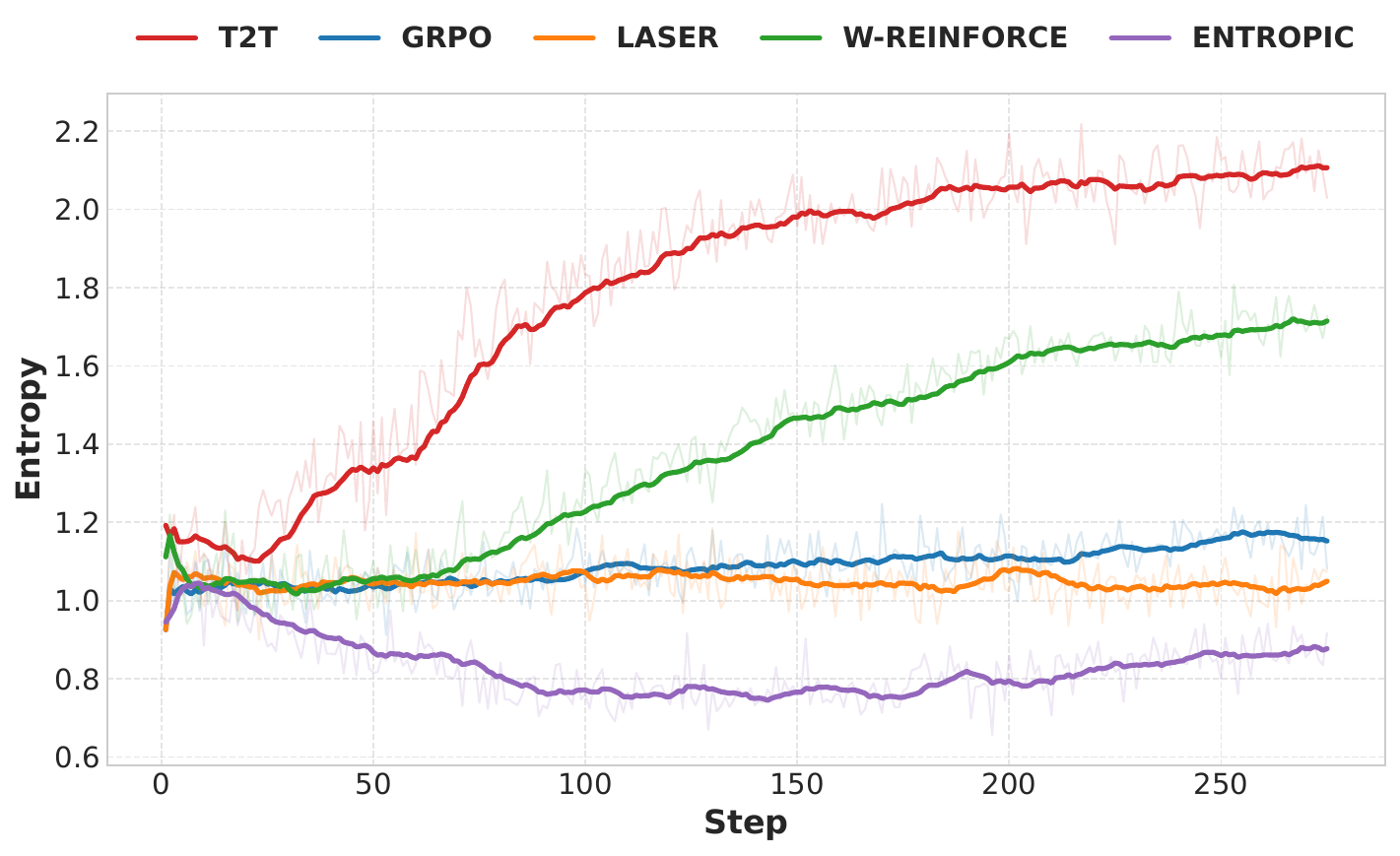} 
  \end{minipage}
  \hfill
  \begin{minipage}{0.49\textwidth}
    \centering
    \includegraphics[width=\linewidth]{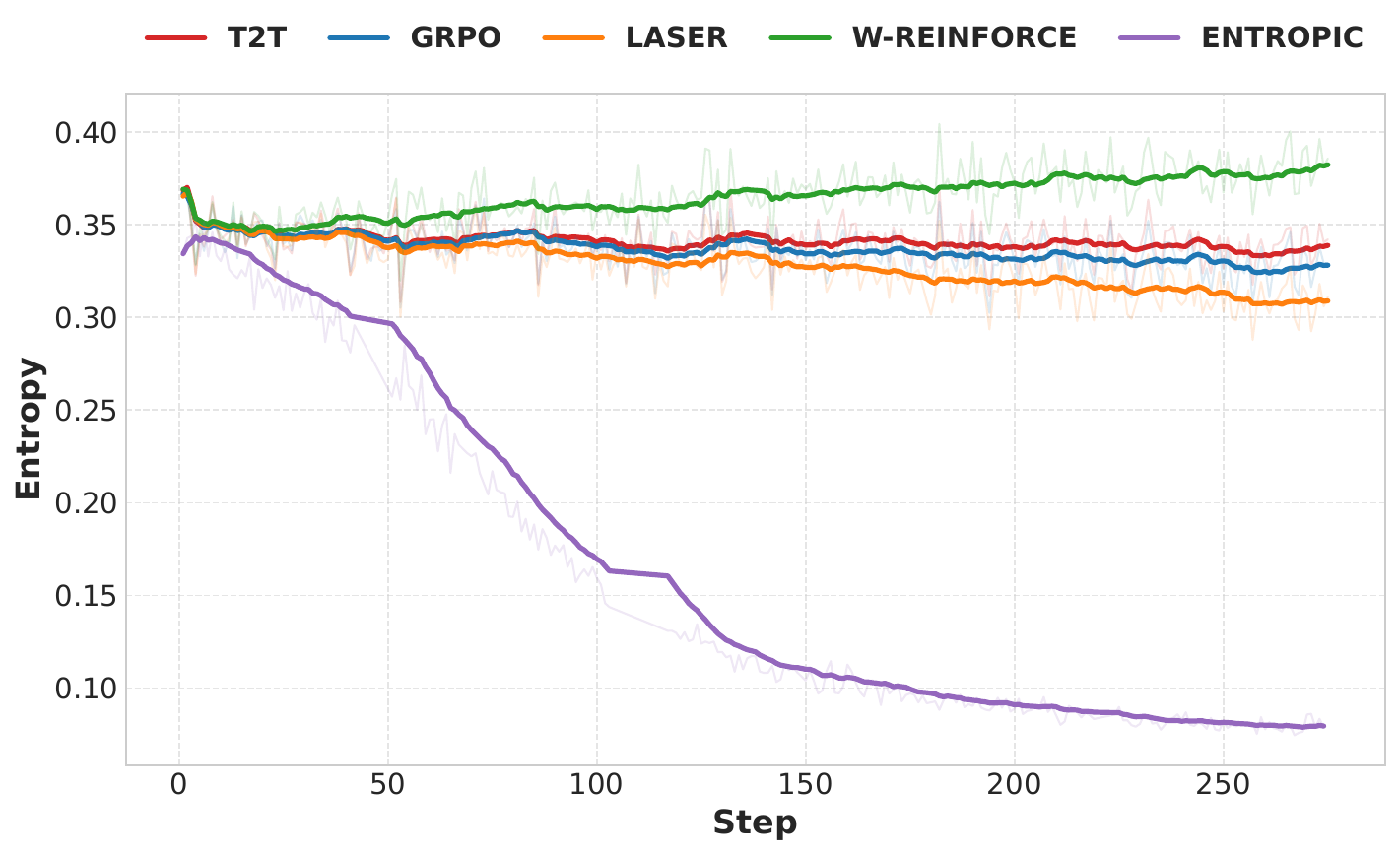} 
  \end{minipage}
  
  \caption{\textbf{Policy Entropy Evolution.} The plot on the left corresponds to \texttt{Qwen2.5-3B}, and the plot on the right corresponds to \texttt{Qwen3-4B}. Regardless of the absolute trend, T2T consistently maintains a higher relative entropy level than the baseline, indicating sustained exploration capabilities.}
  \label{fig:entropy_dynamics}
\end{figure*}

  

\paragraph{Emergent High Entropy without Explicit Regularization.}
A critical observation from training dynamics is the behavior of policy entropy, which varies across model capacities yet consistently favors our method. It is important to emphasize that \textbf{T2T does not employ any explicit entropy regularization terms}.

On \textbf{Qwen2.5-3B} , the standard GRPO baseline exhibits a relatively flat and static entropy curve, suggesting limited variation in its search strategy. In contrast, T2T demonstrates a distinct dynamic: it exhibits a \textbf{steady increase in entropy} that \textbf{eventually converges} to a stable level. This trajectory aligns with our ``Thickening-to-Thinning'' philosophy: the model first expands its search space to explore diverse reasoning paths (Thickening), and then stabilizes as it learns to distinguish correct solutions (Thinning).

On \textbf{Qwen3-4B} , where the baseline shows a tendency towards entropy decay (signaling potential policy collapse), T2T effectively \textbf{counters this drop}, sustaining a significantly higher entropy level throughout training. 

Despite these different underlying dynamics, the common thread is that T2T implicitly incentivizes the model to maintain a broader search space than the baseline. This emergent property directly supports the superior Pass@64 performance observed in our main results.

\paragraph{Accelerated Learning Efficiency.}
Figure~\ref{fig:acc_dynamics} demonstrate that T2T achieves a significantly increase in training accuracy compared to the baseline during training. This rapid improvement stems from the \textbf{synergy of the Thickening-to-Thinning (T2T) strategy}. 
On one hand, the ``Thickening'' phase encourages the model to explore diverse trajectories for hard problems, increasing the likelihood of discovering correct reasoning chains that the baseline might miss. 
On the other hand, once a correct path is found, the ``Thinning'' phase applies a dynamic length penalty, filtering out verbose and irrelevant tokens. This dual mechanism ensures that the model not only finds the correct answers faster but also learns to refine them efficiently, leading to a consistently higher accuracy yield within the same training budget.


\paragraph{Length Modulation Mechanics.}
We further analyze the evolution of response lengths. We observe that T2T adaptively modulates generation length based on intrinsic \textbf{model capability}. Specifically, the method tends to increase length for less capable models (to encourage exploration of reasoning paths), while decreasing it for highly proficient models that have already mastered the tasks (to encourage conciseness). Moreover, a breakdown of positive and negative samples confirms that our method successfully learns a bi-modal strategy: shortening correct responses while elongating incorrect ones. \textbf{Detailed plots and analyses of these length dynamics are provided in Appendix~\ref{app:length_dynamics}.}

\section{Detailed Analysis of Response Length in Training Dynamics}
\label{app:length_dynamics}

In this section, we provide a granular analysis of how T2T modulates the response length across different model scales and correctness outcomes.

\subsection{Adaptive Length Modulation across Model Scales}
\begin{figure*}[t]
  \centering
  \begin{minipage}{0.49\textwidth}
    \centering
    \includegraphics[width=\linewidth]{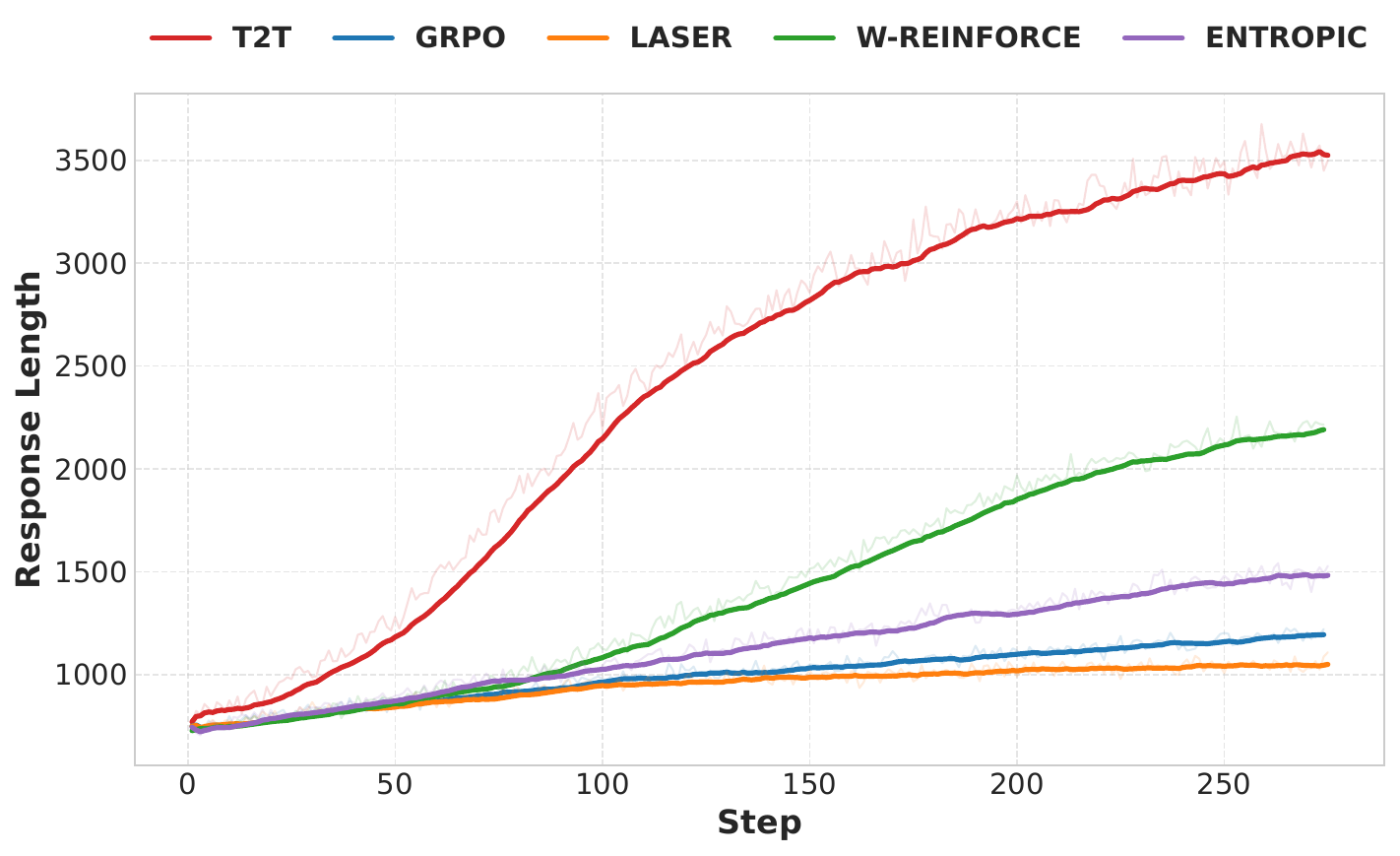} 
  \end{minipage}
  \hfill
  \begin{minipage}{0.49\textwidth}
    \centering
    \includegraphics[width=\linewidth]{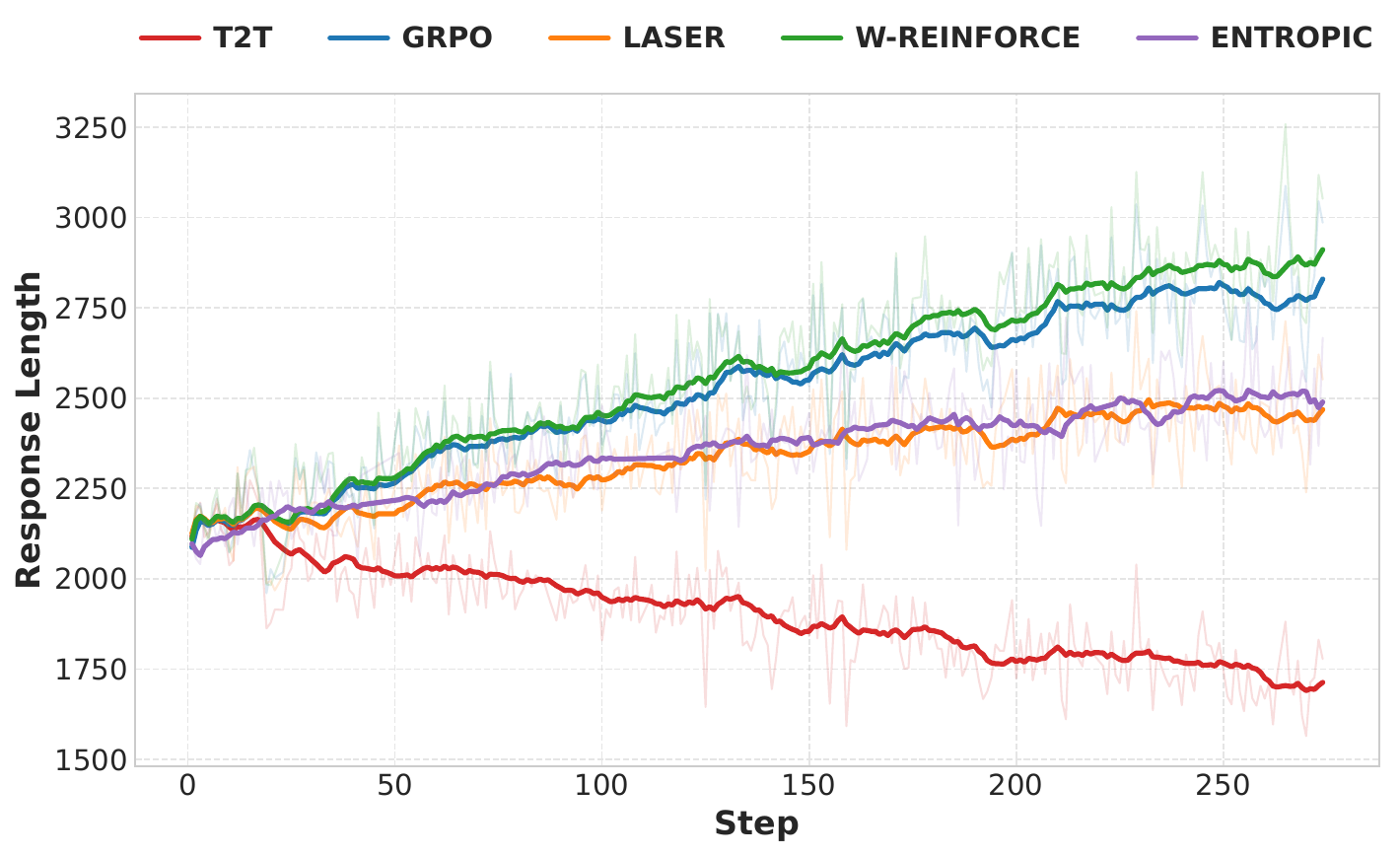} 
  \end{minipage}
  
  \caption{\textbf{Response Length Evolution.} The plot on the left corresponds to \texttt{Qwen2.5-3B}, and the plot on the right corresponds to \texttt{Qwen3-4B}. }
  \label{fig:response_length_dynamics}
\end{figure*}

Figure~\ref{fig:response_length_dynamics} illustrates the evolution of average response length during training. A striking divergence is observed between the two model scales, which strongly validates the adaptive nature of our T2T mechanism.

On \textbf{Qwen2.5-3B}, the response length for T2T increases relative to the initial stage. This suggests that for the smaller model, the training queries pose a higher level of intrinsic difficulty. Consequently, the ``Thickening'' incentive dominates the reward signal, encouraging the model to generate longer chains of thought to bridge reasoning gaps. This aligns with the entropy increase observed in Figure~\ref{fig:entropy_dynamics}, confirming that the model is actively exploring via expansion.
Conversely, on \textbf{Qwen3-4B}, the T2T response length exhibits a downward trend. Since the 4B model possesses stronger intrinsic reasoning capabilities, it achieves high success rates earlier in training. As the model gains confidence, the ``Thinning'' penalty takes precedence, incentivizing the pruning of redundant tokens to converge on more concise solutions.
This contrast highlights a key advantage of T2T: it does not enforce a static length bias (e.g., always longer or always shorter). Instead, it acts as a \textbf{competence-aware regulator}, dynamically allocating the computational budget based on the model's mastery of the task.

\subsection{Disentangling Correct and Incorrect Response Dynamics}

To rigorously verify that the observed length changes are causally linked to response correctness—as dictated by our T2T reward formulation—we decompose the response length evolution into positive (correct) and negative (incorrect) samples. Figure~\ref{fig:length_dynamics} presents this breakdown for the \texttt{Qwen3-14B} model.


\begin{figure*}[t]
  \centering
  
  \begin{minipage}{0.49\textwidth}
    \centering
    \includegraphics[width=\linewidth]{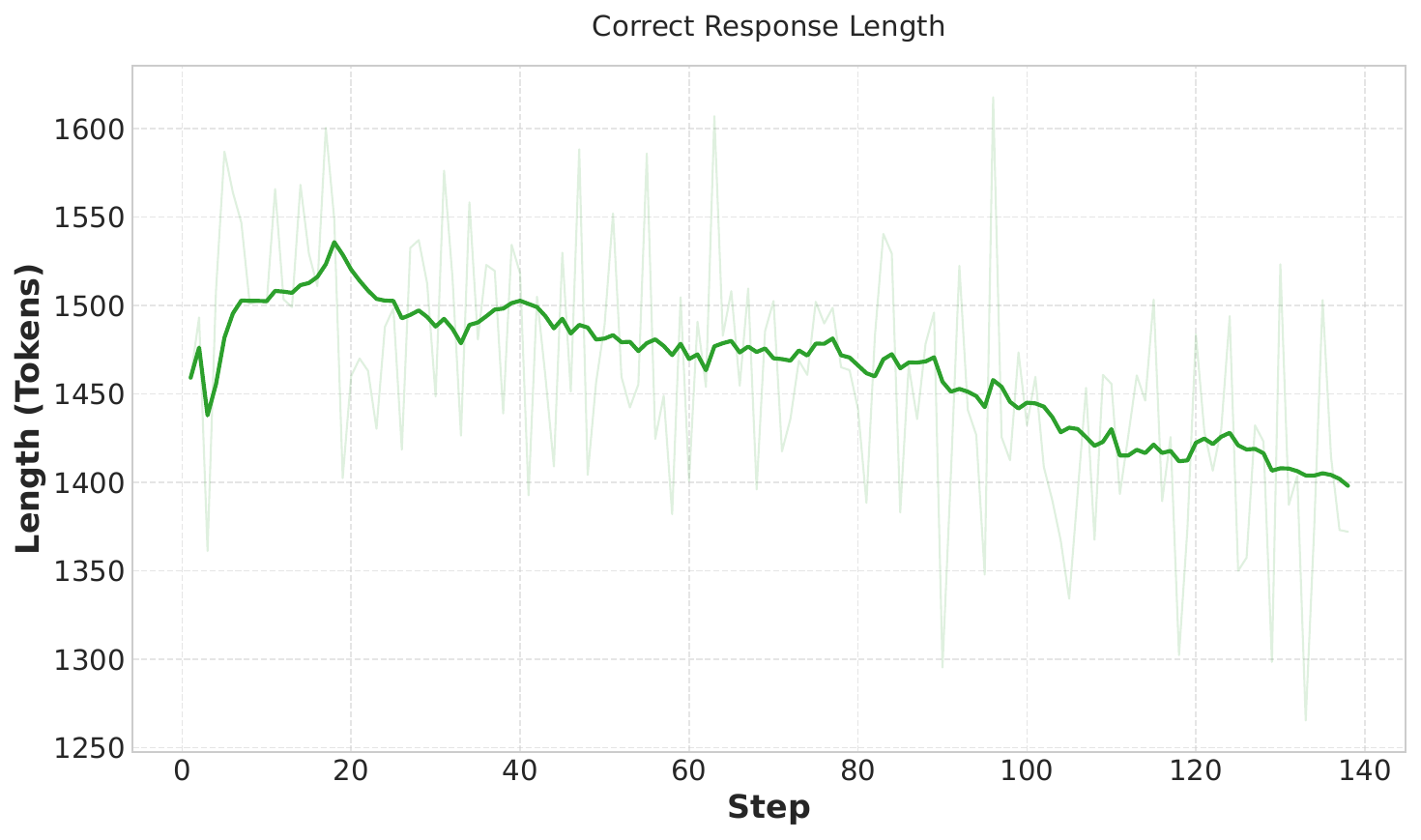}
    \caption*{(a) Correct Responses (Positive)}
  \end{minipage}
  \hfill
  \begin{minipage}{0.49\textwidth}
    \centering
    \includegraphics[width=\linewidth]{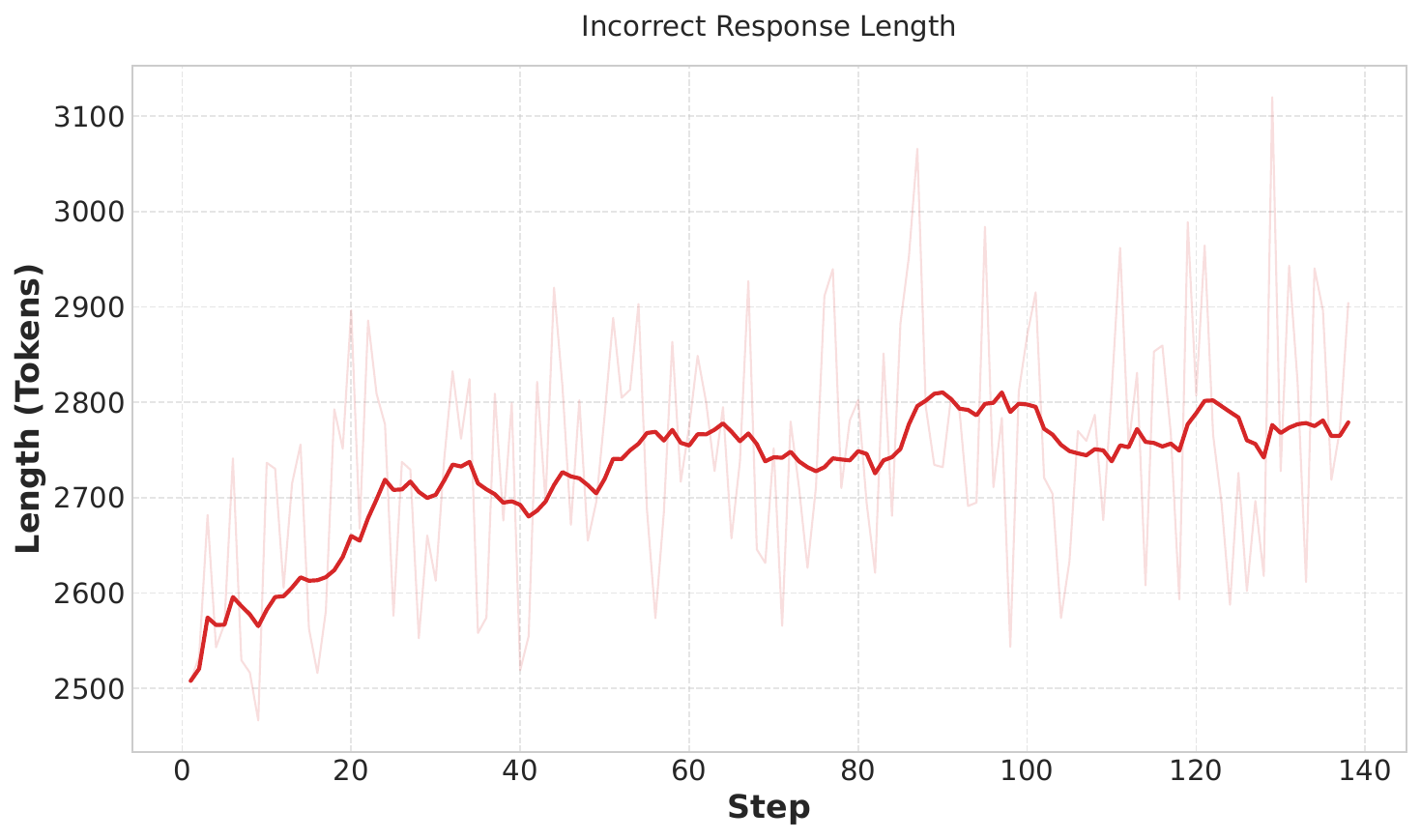}
    \caption*{(b) Incorrect Responses (Negative)}
  \end{minipage}
  
  \caption{\textbf{Divergent Length Dynamics based on Correctness (Qwen3-14B).} Analysis of token length evolution for (a) Correct and (b) Incorrect samples. Correct responses systematically shorten (Thinning) while incorrect responses lengthen (Thickening), confirming the bi-modal optimization strategy of T2T.}
  \label{fig:length_dynamics}
\end{figure*}

As shown in Figure~\ref{fig:length_dynamics}, the model exhibits a striking divergence in behavior depending on the outcome:

\begin{itemize}
    \item \textbf{Thinning for Correctness:} The length of \textbf{positive samples} shows a consistent downward trend. This confirms that when the model successfully solves a problem, the adaptive length penalty effectively encourages the policy to converge towards the most efficient reasoning path.
    \item \textbf{Thickening for Failure:} Conversely, the length of \textbf{negative samples} increases over time. This indicates that when the model fails to reach the correct answer, the reward formulation incentivizes it to generate longer sequences,forcing the model to explore alternative reasoning steps.
\end{itemize}

\paragraph{Impact of Transitioning Samples.}
It is worth noting that the divergence shown in Figure~\ref{fig:length_dynamics} likely \textit{underestimates} the true strength of the T2T mechanism due to the dynamic nature of training. As the model improves (see accuracy curves), many difficult problems that initially resulted in long, incorrect responses eventually transition into the "correct" category. 
Since these newly solved problems often require complex (and thus longer) reasoning chains compared to easy problems, their migration into the positive pool partially offsets the length reduction of the originally correct samples. Similarly, their removal from the negative pool dampens the average length increase of the remaining incorrect samples. 
Therefore, the intrinsic drive of T2T to shorten correct paths and lengthen incorrect exploration is likely even stronger than what the aggregated statistics suggest.

This clear separation of dynamics provides strong empirical evidence that T2T creates a \textbf{bi-modal strategy}: efficient execution for known knowledge, and extensive exploration for unknown territories.

\section{Full Test Set Length Report}
\label{app:test_set_length}
We report the average reasoning trajectory lengths across the full test set (including both correct and incorrect samples) for T2T and vanilla GRPO in Table~\ref{tab:appendix_length_full}.

\begin{table}[h]
\centering
\caption{Average reasoning trajectory lengths (tokens) over the \textbf{full test set} for GRPO and T2T.}
\label{tab:appendix_length_full}
\resizebox{0.8\linewidth}{!}{
\begin{tabular}{llcccc}
\toprule
Model & Method & MATH500 & AIME24 & AIME25 & AMC23 \\
\midrule
\multirow{2}{*}{DS-1.5B} 
& GRPO & 3010 & 3674 & 4125 & 4289 \\
& T2T  & 2842 & 3956 & 4412 & 4573 \\
\midrule
\multirow{2}{*}{Llama3.2-3B} 
& GRPO & 1009 & 2709 & 2088 & 2247 \\
& T2T  & 1146 & 2984 & 2357 & 2110 \\
\midrule
\multirow{2}{*}{Qwen2.5-3B} 
& GRPO & 1764 & 2386 & 2917 & 3164 \\
& T2T  & 1984 & 2612 & 3245 & 3418 \\
\midrule
\multirow{2}{*}{Qwen3-4B} 
& GRPO & 1269 & 4824 & 4855 & 2251 \\
& T2T  & 1124 & 4156 & 4287 & 2034 \\
\midrule
\multirow{2}{*}{Qwen3-14B} 
& GRPO & 1384 & 4142 & 4370 & 2418 \\
& T2T  & 1217 & 3784 & 4276 & 2156 \\
\bottomrule
\end{tabular}
}
\end{table}

As discussed, T2T’s length control is not a static compression objective but a dynamic mechanism aligned with the model’s mastery state: for problems the model struggles with, T2T encourages expansive exploration, leading to longer reasoning trajectories; for problems the model has mastered, it incentivizes concise reasoning, leading to shorter trajectories. Consequently, the average length over an entire dataset depends directly on the distribution of problem difficulty and the model’s overall mastery of the test set at evaluation time.

T2T’s length modulation is adaptive and context-dependent, rather than a one-size-fits-all compression strategy. We would like to clarify again that length control is not our major objective but an instrument to implement T2T dynamics. Our core objective is to fundamentally enhance the model's reasoning capabilities by simulating human learning dynamics. 

\section{Hyperparameter Sensitivity Analysis}
\label{app:alpha_sensitivity}

A practical reinforcement learning algorithm should be robust to hyperparameter choices and not require exhaustive tuning. 
To verify the sensitivity of T2T to the length scaling factor $\alpha$, we conducted experiments on the \texttt{Qwen-14B} model comparing two distinct settings: $\alpha=0.1$ and $\alpha=0.2$.

The parameter $\alpha$ controls the magnitude of the length incentive relative to the correctness reward. As derived in Section ~\ref{sec:method}, the theoretical constraint $\alpha < 0.5$ ensures that any correct response yields a higher reward than any incorrect response, regardless of length. Within this valid range, $\alpha$ merely adjusts the "strength" of the shaping signal rather than its direction.

The results are summarized in Table~\ref{tab:alpha_sensitivity}. We observe that the performance differences between $\alpha=0.1$ and $\alpha=0.2$ are marginal across benchmarks. 
For instance, the Pass@1 score on MATH-500 fluctuates by less than 1\%, and the exploration capability (Pass@64) remains consistently high. 
This indicates that our method is \textbf{robust to the selection of $\alpha$}: as long as the parameter is set within a reasonable range (sufficiently large to provide a gradient signal but satisfying $\alpha < 0.5$), the T2T mechanism effectively modulates the reasoning length without requiring precise tuning. 

\begin{table*}[h]
    \centering
    \caption{\textbf{Sensitivity Analysis of $\alpha$ on Qwen3-14B.} We compare the performance with $\alpha=0.1$ and $\alpha=0.2$. The method demonstrates high stability, showing that precise tuning of the scaling factor is not critical for success.}
    \label{tab:alpha_sensitivity}
    \vspace{5pt}
     \resizebox{0.9\textwidth}{!}{%
    \begin{tabular}{l|cc|cc|cc|cc}
        \toprule
        \multirow{2}{*}{\textbf{Configuration}} & \multicolumn{2}{c|}{\textbf{MATH-500}} & \multicolumn{2}{c}{\textbf{AIME'24}} & \multicolumn{2}{c}{\textbf{AIME'25}}& \multicolumn{2}{c}{\textbf{AMC'23}}\\
        & Pass@1 & Pass@64 & Pass@1 & Pass@64 & Pass@1 & Pass@64 & Pass@1 & Pass@64 \\
        \midrule
        $\alpha = 0.1$ & 85.0 & 92.2 & 39.9 & 60.1 &  24.9& 58.3 & 80.6 & 98.2 \\ 
        $\alpha = 0.2$ & 85.1 & 92.7 & 39.6 & 59.6 & 24.5 & 57.9 & 81.4 & 98.4 \\ 
        \bottomrule
            \end{tabular}}
\end{table*}

\paragraph{Analysis of Training Dynamics.}
To further understand how the scaling factor $\alpha$ influences the optimization process, we visualize the evolution of training accuracy and response length in Figure~\ref{fig:alpha_dynamics}.

\begin{figure}[t]
  \centering
  \begin{minipage}{0.49\textwidth}
    \centering
    \includegraphics[width=\linewidth]{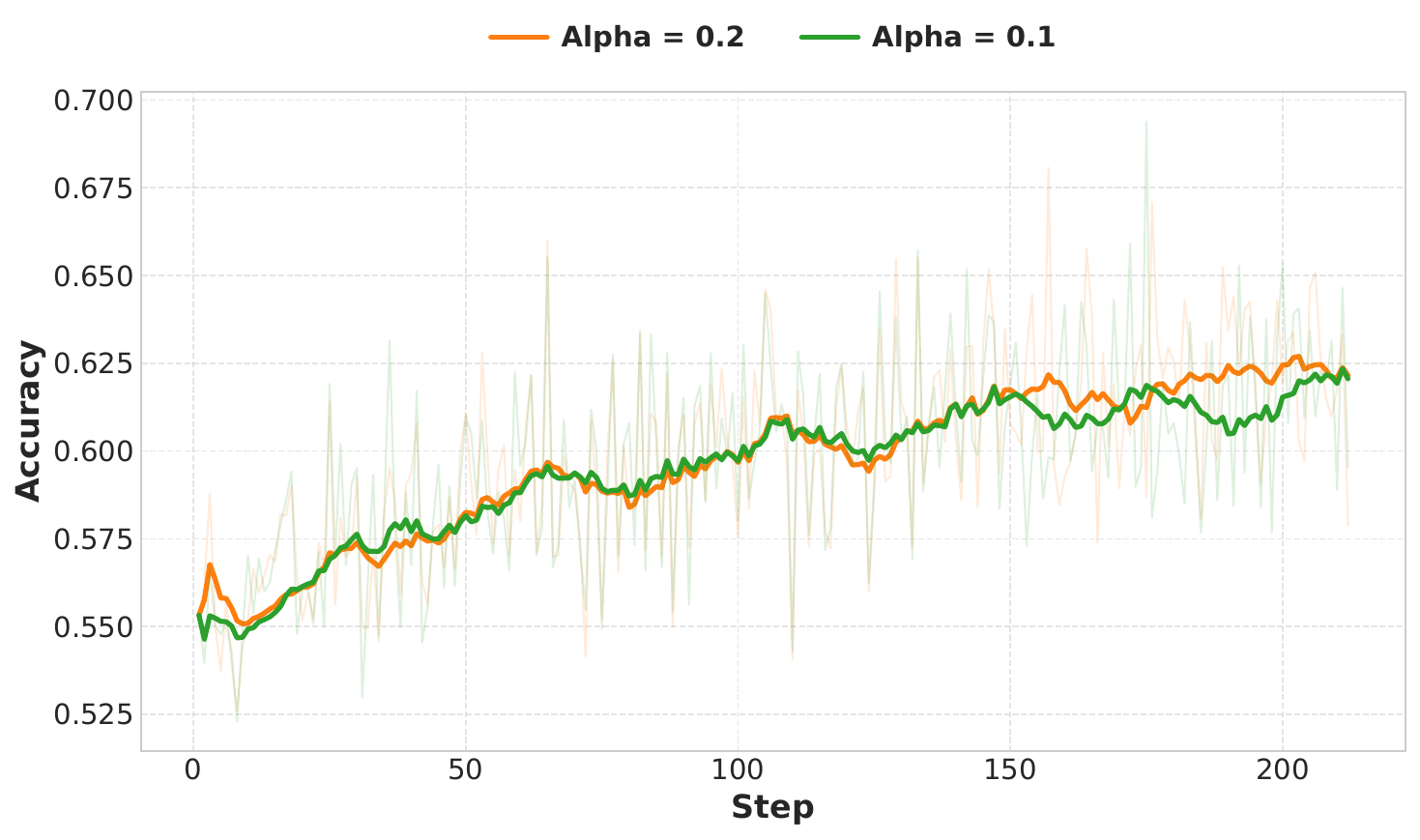}
    \caption*{(a) Training Accuracy}
  \end{minipage}
  \hfill
  \begin{minipage}{0.49\textwidth}
    \centering
    \includegraphics[width=\linewidth]{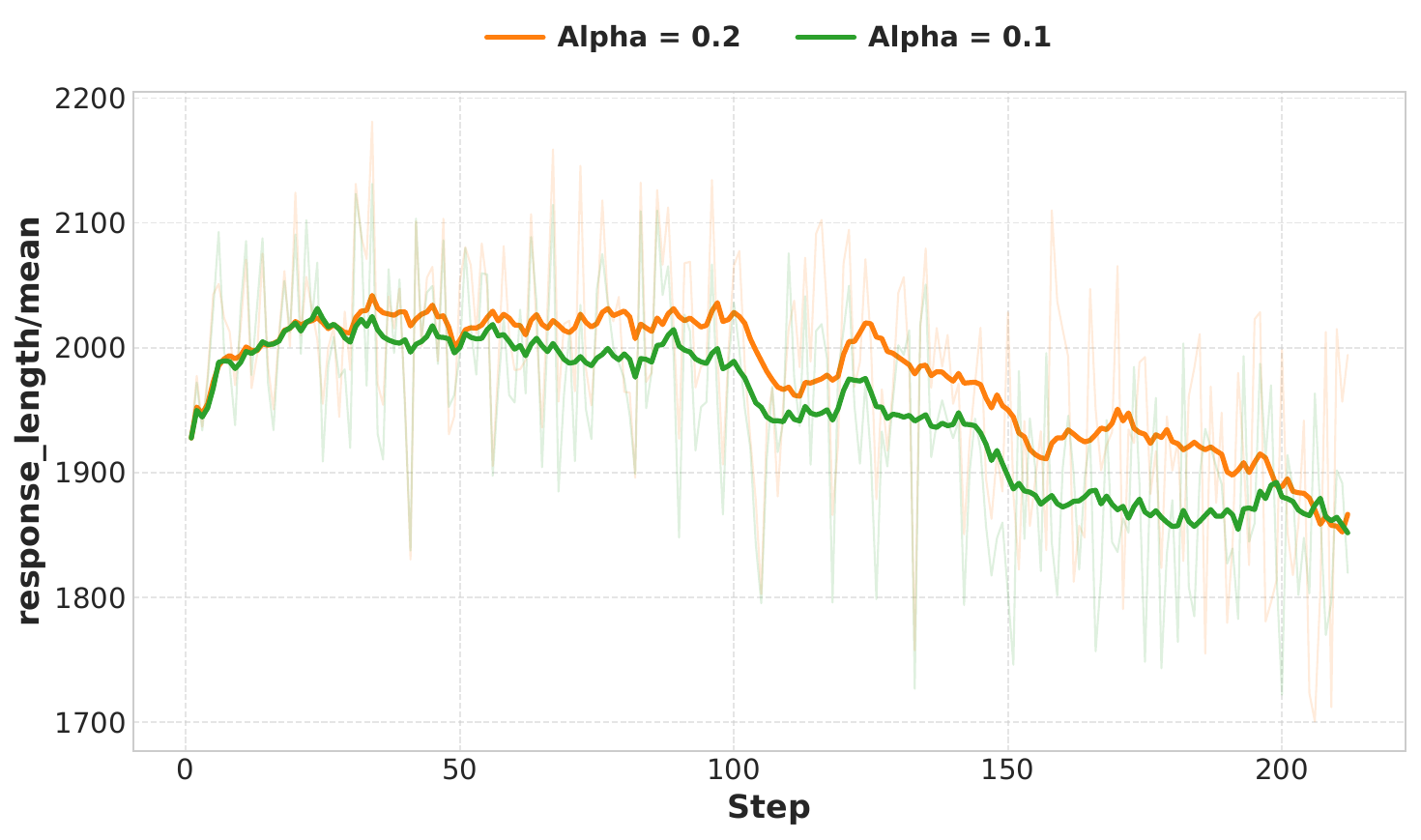}
    \caption*{(b) Average Response Length}
  \end{minipage}
  \caption{\textbf{Training Dynamics with different $\alpha$.} Comparing $\alpha=0.1$ and $\alpha=0.2$ on Qwen-14B. (a) Both settings achieve similar convergence rates and final accuracy. (b) The response length dynamics show that a larger $\alpha$ induces a slightly stronger modulation in token generation, yet both stabilize effectively.}
  \label{fig:alpha_dynamics}
\end{figure}

As illustrated in Figure~\ref{fig:alpha_dynamics}, the model exhibits highly consistent behavior across different shaping magnitudes ($\alpha=0.1$ and $\alpha=0.2$). 
Both the learning trajectories (Figure~\ref{fig:alpha_dynamics}a) and the response length dynamics (Figure~\ref{fig:alpha_dynamics}b) follow nearly identical patterns. 
This stability confirms that the proposed T2T method is robust to the hyperparameter $\alpha$, consistently guiding the model towards an efficient reasoning policy without requiring precise tuning.

\section{Full Experimental Results}
\subsection{Detailed Experimental Results on Mathematical Dataset}
\label{app:full_results}
In this section, we provide the comprehensive performance metrics for all evaluated models across the four benchmarks. To thoroughly analyze the exploration capabilities of different methods, we report \textbf{Pass@$k$} for $k \in \{1, 2, 4, 8, 16, 32, 64\}$. Tables~\ref{tab:full_math500}, \ref{tab:full_aime24}, \ref{tab:full_aime25}, and \ref{tab:full_AMC'23} present the results for MATH-500, AIME'24, AIME'25, and AMC'23, respectively.

\begin{table}[H]
  \caption{\textbf{Full Results on MATH-500}. Comparison of Pass@$k$ ($k=1,2,4,8,16,32,64$).}
  \label{tab:full_math500}
  \centering
  \resizebox{\textwidth}{!}{
  \begin{tabular}{c|l|ccccccc}
    \toprule
    \textbf{Base Model} & \textbf{Method} & \textbf{Pass@1} & \textbf{Pass@2} & \textbf{Pass@4} & \textbf{Pass@8} & \textbf{Pass@16} & \textbf{Pass@32} & \textbf{Pass@64} \\
    \midrule

    \multirow{6}{*}{DeepSeek-R1-Distill-Qwen-1.5B}
    & Initial Model   & 66.0 & 72.2 & 76.6 & 79.7 & 82.1 & 83.8 & 85.0 \\
    & GRPO            & \textbf{68.9} & \textbf{73.8} & 77.3 & 79.9 & 81.8 & 83.2 & 84.0 \\
    & LASER           & 66.5 & 72.6 & 77.0 & 80.2 & 82.5 & 84.1 & 85.4 \\
    & W-REINFORCE    & \underline{67.5} & \underline{73.5} & \underline{77.5} & \underline{80.5} & \underline{82.8} & \underline{84.5} & \underline{85.8} \\
    & EntroPIC        & 66.2 & 72.4 & 76.8 & 79.5 & 81.9 & 83.5 & 84.8 \\
    & \textbf{T2T}    & 66.9 & 73.3 & \textbf{77.8} & \textbf{81.0} & \textbf{83.4} & \textbf{85.0} & \textbf{86.1} \\
    \midrule

    \multirow{3}{*}{Llama-3.2-3B-Instruct} 
    & Initial Model   & 24.9 & 26.8 & 28.4 & 29.8 & 31.0 & 31.9 & 32.5 \\
    & GRPO            & \underline{34.3} & \underline{36.9} & \underline{39.3} & \underline{41.4} & \underline{43.0} & \underline{44.1} & \underline{45.0} \\
    & \textbf{T2T}    & \textbf{36.2} & \textbf{38.9} & \textbf{41.3} & \textbf{43.4} & \textbf{45.1} & \textbf{46.5} & \textbf{47.6}\\
    \midrule

    \multirow{6}{*}{Qwen2.5-3B} 
    & Initial Model      & 40.6 & 43.4 & 44.9 & 46.7 & 47.9 & 48.7 & 49.1 \\
    & GRPO      & 47.4 & \underline{51.6} & \underline{53.3}  & 56.6 & 58.3 & 59.5 & 59.9 \\
    & LASER     & \underline{48.4} & 50.9 & 53.2 & 55.2 & 56.5 & 57.3 & 57.9 \\
    & W-REINFORCE    & 46.5 & 51.1 & 53.1 & \underline{56.7} & \underline{58.6} & 59.9 & 60.2 \\
    & EntroPIC  & 44.6 & 49.0 & 52.7 & 55.8 & 58.1 & \underline{60.0} & \underline{61.3} \\
    & \textbf{T2T} & \textbf{49.8} & \textbf{53.1} & \textbf{56.3} & \textbf{58.9} & \textbf{61.1} & \textbf{63.0} & \textbf{64.6} \\
    \midrule

    \multirow{6}{*}{Qwen3-4B} 
    & Initial Model   & 80.2 & 83.3 & 85.7 & 87.6 & 89.1 & \underline{90.3} & \textbf{91.2} \\
    & GRPO            & \textbf{81.2} & \underline{84.1} & 86.3 & 88.0 & 89.1 & 89.9 & 90.5 \\
    & LASER           & 80.6 & 83.9 & 86.4 & \underline{88.3} & \underline{89.4} & 90.2 & 90.7 \\
    & W-REINFORCE     & \underline{80.9} & \underline{84.1} & \underline{86.5} & 88.1 & 89.2 & 90.0 & 90.6\\
    & EntroPIC        & 80.5 & 83.8 & 86.1 & 87.9 & 89.0 & 90.1 & 90.8 \\
    & \textbf{T2T}    & \textbf{81.2} & \textbf{84.5} & \textbf{86.9} & \textbf{88.6} & \textbf{89.8} & \textbf{90.6} & \underline{91.1} \\
    \midrule
    
    \multirow{3}{*}{Qwen3-14B} 
    & Initial Model   & 82.2 & 84.0 & 85.3 & 86.1 & 86.6 & 87.0 & 87.2 \\
    & GRPO            & \underline{83.3} & \underline{85.1} & \underline{86.4} & \underline{87.3} & \underline{88.0} & \underline{88.5} & \underline{88.8} \\
    & \textbf{T2T}    & \textbf{85.1} & \textbf{87.5} & \textbf{89.1} & \textbf{90.4} & \textbf{91.4} & \textbf{92.2} & \textbf{92.7}\\
    \bottomrule
  \end{tabular}
  }
\end{table}

\begin{table}[H]
  \caption{\textbf{Full Results on AIME'24}. Comparison of Pass@$k$ ($k=1,2,4,8,16,32,64$).}
  \label{tab:full_aime24}
  \centering
  \resizebox{\textwidth}{!}{
  \begin{tabular}{c|l|ccccccc}
    \toprule
    \textbf{Base Model} & \textbf{Method} & \textbf{Pass@1} & \textbf{Pass@2} & \textbf{Pass@4} & \textbf{Pass@8} & \textbf{Pass@16} & \textbf{Pass@32} & \textbf{Pass@64} \\
    \midrule

    \multirow{6}{*}{DeepSeek-R1-Distill-Qwen-1.5B}
    & Initial Model   & 20.5 & \textbf{25.2} & \textbf{30.0} & \underline{34.9} & 39.7 & 43.7 & 46.9 \\
    & GRPO            & 20.6 & 22.5 & 27.1 & 32.7 & 38.8 & 44.1 & 47.7 \\
    & LASER           & 20.8 & 22.9 & 27.6 & 33.2 & 39.4 & 44.8 & 48.3 \\
    & W-REINFORCE     & 20.4 & 22.1 & 26.8 & 32.0 & 38.1 & 43.2 & 46.8 \\
    & EntroPIC        & \underline{21.0} & 23.5 & 28.5 & 34.5 & \underline{40.8} & \underline{45.9} & \underline{49.5} \\
    & \textbf{T2T}    & \textbf{21.3} & \underline{23.7} & \underline{29.9} & \textbf{36.2} & \textbf{42.1} & \textbf{47.3} & \textbf{50.7} \\
    \midrule

    \multirow{3}{*}{Llama-3.2-3B-Instruct} 
    & Initial Model   & 12.6 & 12.9 & 13.0 & 13.1 & 13.2 & 13.3 & 13.3 \\
    & GRPO            & \underline{14.7} & \underline{16.0} & \underline{16.9} & \underline{17.5} & \underline{18.1} & \underline{19.4} & \underline{20.9} \\
    & \textbf{T2T}    & \textbf{15.2} & \textbf{17.5} & \textbf{19.3} & \textbf{21.0} & \textbf{22.4} & \textbf{23.9} & \textbf{25.4}\\
    \midrule

    \multirow{6}{*}{Qwen2.5-3B} 
    & Initial Model      & 7.1 & 7.4 & 8.1 & 9.1 & 10.8 & 12.7 & 14.7 \\
    & GRPO      & 4.8 & 5.8 & 7.0 & 8.4 & 10.0 & 12.1 & 14.7 \\
    & LASER     & 6.0 & 6.9 & 7.7 & 8.7 & 9.9 & 11.2 & 12.2 \\
    & W-REINFORCE    & \textbf{12.6} & \textbf{14.5} & \textbf{16.4} & \textbf{17.6} & \textbf{18.8} & \textbf{20.3} & \textbf{22.8}\\
    & EntroPIC  & 7.2 & 7.9 & 8.9 & 10.4 & 12.9 & 15.9 & 18.3 \\
    & \textbf{T2T}  & \underline{9.6} & \underline{11.3} & \underline{13.6} & \underline{15.8} & \underline{18.1} & \underline{20.0} & \underline{21.7}  \\
    \midrule

    \multirow{6}{*}{Qwen3-4B} 
    & Initial Model      & 22.2 & 24.2 & 26.1 & 28.2 & 31.2 & 34.2 & 37.0 \\
    & GRPO      & \underline{29.3} & \underline{32.2} & \underline{35.7} & \underline{39.6} & \underline{44.7} & 47.0 & 48.6 \\
    & LASER     & 28.2 & 31.0 & 32.6 & 34.5 & 39.8 & 43.3 & 45.4 \\
    & W-REINFORCE    & 28.5 & 31.1 & 34.1 & 37.6 & 44.2 & \underline{48.2} & \underline{50.7} \\
    & EntroPIC  & 28.8 & 31.5 & 33.8 & 36.2 & 42.5 & 45.8 & 47.9  \\
    & \textbf{T2T} & \textbf{33.4} & \textbf{37.6} & \textbf{41.9} & \textbf{46.1} & \textbf{50.7} & \textbf{54.5} & \textbf{56.3} \\
    \midrule

    \multirow{3}{*}{Qwen3-14B} 
    & Initial Model   & 27.1 & 32.1 & 35.9 & 38.7 & 41.4 & 44.6 & 48.3 \\
    & GRPO            & \underline{29.1} & \underline{34.3} & \underline{39.3} & \underline{44.4} & \underline{49.3} & \underline{53.5} & \underline{56.1} \\
    & \textbf{T2T}    & \textbf{39.6} & \textbf{45.7} & \textbf{50.6} & \textbf{53.6} & \textbf{55.4} & \textbf{57.3} & \textbf{59.6}\\
    \bottomrule
  \end{tabular}
  }
\end{table}

\begin{table}[H]
  \caption{\textbf{Full Results on AIME'25}. Comparison of Pass@$k$ ($k=1,2,4,8,16,32,64$).}
  \label{tab:full_aime25}
  \centering
  \resizebox{\textwidth}{!}{
  \begin{tabular}{c|l|ccccccc}
    \toprule
    \textbf{Base Model} & \textbf{Method} & \textbf{Pass@1} & \textbf{Pass@2} & \textbf{Pass@4} & \textbf{Pass@8} & \textbf{Pass@16} & \textbf{Pass@32} & \textbf{Pass@64} \\
    \midrule

    \multirow{6}{*}{DeepSeek-R1-Distill-Qwen-1.5B}
    & Initial Model   & 14.8 & 17.5 & 19.0 & 20.5 & 22.6 & 25.8 & 28.2 \\
    & GRPO            & 14.8 & 17.6 & \underline{20.2} & \underline{22.4} & \underline{24.4} & \underline{26.8} & \textbf{28.9} \\
    & LASER           & 14.9 & 17.8 & 19.5 & 21.2 & 23.1 & 26.1 & 28.4 \\
    & W-REINFORCE     & 14.7 & 17.4 & 19.1 & 20.8 & 22.9 & 25.5 & 28.0 \\
    & EntroPIC        & \underline{15.0} & \underline{17.9} & 20.0 & 21.9 & 23.5 & 26.4 & 28.5 \\
    & \textbf{T2T}    & \textbf{15.2} & \textbf{18.4} & \textbf{21.4} & \textbf{23.8} & \textbf{25.7} & \textbf{27.5} & \underline{28.8} \\
    \midrule

    \multirow{3}{*}{Llama-3.2-3B-Instruct} 
    & Initial Model   & 0.0 & 0.0 & 0.0 & 0.0 & 0.0 & 0.0 & 0.0 \\
    & GRPO            & \underline{3.2} & \underline{3.8} & \underline{4.2} & \underline{4.9} & \underline{5.8} & \underline{6.4} & \underline{6.6} \\
    & \textbf{T2T}    & \textbf{3.6} & \textbf{3.8} & \textbf{4.3} & \textbf{5.0} & \textbf{5.9} & \textbf{6.6} & \textbf{6.6}\\
    \midrule

    \multirow{6}{*}{Qwen2.5-3B} 
    & Initial Model      & 3.4 & 3.4 & 3.5 & 3.6 & 3.8 & 4.1 & 4.6 \\
    & GRPO      & \textbf{7.7} & \textbf{8.5} & \textbf{9.7} & \underline{11.1} & 12.5 & 14.2 & 15.6 \\
    & LASER     & 4.6 & 6.9 & 7.7 & 8.7 & 9.9 & 11.1 & 12.2 \\
    & W-REINFORCE     & \underline{5.8} & \underline{7.2} & 8.8 & 10.5 & \underline{12.6} & \underline{14.4} & \underline{15.9} \\
    & EntroPIC  & 3.6 & 4.0 & 4.5 & 5.4 & 6.7 & 8.1 & 9.2 \\
    & \textbf{T2T} & 5.4 & 7.1 & \underline{9.3} & \textbf{11.7} & \textbf{13.8} & \textbf{15.2} & \textbf{16.7} \\
    \midrule
    
    \multirow{6}{*}{Qwen3-4B} 
    & Initial Model      & 14.7 & 17.5 & 21.6 & 26.4 & 31.0 & 33.9 & 35.4 \\
    & GRPO      & \underline{21.3} & \underline{24.6} & \underline{29.2} & 33.8 & 39.3 & 43.9 & 48.5 \\
    & LASER     & \underline{21.3} & 23.3 & 27.0 & 32.0 & 37.7 & 43.7 & 46.9 \\
    & W-REINFORCE    & 20.8 & 24.3 & 27.7 & \underline{35.4} & \underline{40.2} & \underline{45.2} & \textbf{50.9} \\
    & EntroPIC   & 21.0 & 24.0 & 28.5 & 33.5 & 39.0 & 44.5 & 48.0  \\
    & \textbf{T2T} & \textbf{23.0} & \textbf{26.7} & \textbf{31.4} & \textbf{36.1} & \textbf{41.2} & \textbf{46.3} & \underline{49.6} \\
    \midrule

    \multirow{3}{*}{Qwen3-14B} 
    & Initial Model   & 19.6 & 22.3 & 26.3 & 31.7 & 37.7 & 43.2 & 47.3 \\
    & GRPO            & \underline{22.4} & \underline{24.5} & \underline{28.0} & \underline{33.3} & \underline{39.2} & \underline{45.7} & \underline{51.5} \\
    & \textbf{T2T}    & \textbf{24.5} & \textbf{29.1} & \textbf{34.7} & \textbf{41.1} & \textbf{48.1} & \textbf{54.2} & \textbf{57.9}\\
    \bottomrule
  \end{tabular}
  }
\end{table}

\begin{table}[H]
  \caption{\textbf{Full Results on AMC'23}. Comparison of Pass@$k$ ($k=1,2,4,8,16,32,64$).}
  \label{tab:full_AMC'23}
  \centering
  \resizebox{\textwidth}{!}{
  \begin{tabular}{c|l|ccccccc}
    \toprule
    \textbf{Base Model} & \textbf{Method} & \textbf{Pass@1} & \textbf{Pass@2} & \textbf{Pass@4} & \textbf{Pass@8} & \textbf{Pass@16} & \textbf{Pass@32} & \textbf{Pass@64} \\
    \midrule

    \multirow{6}{*}{DeepSeek-R1-Distill-Qwen-1.5B} 
    & Initial Model      & 50.9 & 59.7 & 67.7 & 74.6 & 79.7 & 83.7 & 86.4 \\
    & GRPO      & 57.3 & 65.3 & 73.5 & 80.6 & \underline{86.1} & \textbf{89.9} & \textbf{92.7} \\
    & LASER     & 57.8 & 65.8 & 73.9 & 80.8 & 85.8 & 89.2 & 91.5 \\
    & W-REINFORCE    & 58.2 & 66.1 & 74.0 & 81.0 & 86.0 & 89.4 & 91.6\\
    & EntroPIC  & \underline{58.5} & \underline{66.4} & \underline{74.2} & \underline{81.1} & 85.9 & 89.3 & 91.4 \\
    & \textbf{T2T} & \textbf{59.0} & \textbf{66.9} & \textbf{74.5} & \textbf{81.3} & \textbf{86.2} & \underline{89.5} & \underline{91.7} \\
    \midrule

    \multirow{3}{*}{Llama-3.2-3B-Instruct} 
    & Initial Model   & 15.4 & 17.3 & 19.3 & 20.9 & 22.2 & 22.4 & 22.5 \\
    & GRPO            & \underline{29.2} & \underline{32.4} & \underline{35.2} & \underline{37.4} & \underline{39.5} & \underline{41.8} & \underline{43.7} \\
    & \textbf{T2T}    & \textbf{32.0} & \textbf{36.0} & \textbf{39.8} & \textbf{43.7} & \textbf{46.6} & \textbf{48.3} & \textbf{49.1}\\
    \midrule

    \multirow{6}{*}{Qwen2.5-3B} 
    & Initial Model      & 43.9 & 47.7 & 50.8 & 53.4 & 55.8 & 57.8 & 59.0 \\
    & GRPO      & 49.1 & 52.9 & 55.8 & 57.5 & 59.2 & 61.3 & 63.4 \\
    & LASER     & 51.9 & 54.5 & 57.1 & 60.0 & 62.4 & 64.6 & 66.2\\
    & NSR        & 52.5 & 56.6 & 60.4 & 63.7 & 66.5 & 68.8 & 70.9 \\
    & EntroPIC   & \underline{54.7} & \underline{59.0} & \underline{62.7} & \underline{65.9} & \underline{68.7} & \underline{70.6} & \underline{71.9} \\
    & \textbf{T2T}  & \textbf{57.0} & \textbf{61.6} & \textbf{65.6} & \textbf{69.8} & \textbf{73.2} & \textbf{75.8} & \textbf{77.7} \\
    \midrule

    \multirow{6}{*}{Qwen3-4B} 
    & Initial Model      & 71.4 & 76.3 & 80.0 & 83.3 & 85.7 & 90.6 & 92.4 \\
    & GRPO      & 75.8 & \underline{79.1} & \underline{82.6} & 87.3 & 88.9 & 89.7 & 91.2 \\
    & LASER     & \underline{76.1} & 78.1 & 80.1 & 86.4 & 88.8 & 90.2 & 91.1 \\
    & W-REINFORCE    & 75.8 & 78.7 & 81.2 & \underline{88.3} & \underline{89.1} & \underline{92.6} & \underline{93.9} \\
    & EntroPIC  & 75.9 & 78.5 & 81.5 & 87.0 & 88.5 & 90.0 & 91.5 \\
    & \textbf{T2T} & \textbf{78.0} & \textbf{81.7} & \textbf{85.1} & \textbf{90.0} & \textbf{92.0} & \textbf{93.5} & \textbf{95.3} \\
    \midrule

    \multirow{3}{*}{Qwen3-14B} 
    & Initial Model   & 74.6 & 80.1 & 84.6 & 87.5 & 89.0 & 90.0 & 91.2 \\
    & GRPO            & \underline{80.7} & \underline{85.1} & \underline{87.4} & \underline{89.4} & \underline{91.6} & \underline{93.5} & \underline{94.6} \\
    & \textbf{T2T}    & \textbf{81.4} & \textbf{86.7} & \textbf{90.0} & \textbf{92.2} & \textbf{94.4} & \textbf{96.7} & \textbf{98.4}\\
    \bottomrule
  \end{tabular}
  }
\end{table}

\subsection{Detailed BBH Experimental Results}
\label{app:bbh_full_results}

In this section, we present the comprehensive breakdown of the \textbf{Big Bench Hard (BBH)} benchmark results in Table~\ref{tab:bbh_3b} and Table~\ref{tab:bbh_4b}. We report the exact scores for all 27 subtasks across \texttt{Qwen2.5-3B} and \texttt{Qwen3-4B}.

\begin{table}[H]
  \caption{\textbf{Full BBH Results on Qwen2.5-3B.}  }
  \label{tab:bbh_3b}
  \centering
  \renewcommand{\arraystretch}{1.15}
  \resizebox{0.95\textwidth}{!}{
  \begin{tabular}{l|c|ccccc}
    \toprule
    \textbf{Subtask} & \textbf{Initial Model} & \textbf{GRPO} & \textbf{LASER} & \textbf{W-REINFORCE} & \textbf{EntroPIC} & \textbf{T2T (Ours)} \\
    \midrule
    Boolean Expressions & 78.00 & 82.00 & 82.40 & \textbf{82.80} & 76.80 & 82.40 \\
    Causal Judgement & 59.89 & 60.96 & 60.96 & 60.43 & 58.82 & \textbf{61.50} \\
    Date Understanding & 59.60 & \textbf{60.40} & 60.00 & 59.20 & 58.40 & 58.80 \\
    Disambiguation QA & 58.40 & 58.40 & 58.40 & 58.40 & 57.60 & 58.40 \\
    Dyck Languages & 4.00 & 3.60 & 3.60 & 4.00 & 4.00 & 4.00 \\
    Formal Fallacies & 54.40 & 55.20 & 55.60 & 52.80 & \textbf{56.40} & 52.00 \\
    Geometric Shapes & 32.80 & \textbf{33.20} & \textbf{33.20} & 32.80 & 31.60 & 32.00 \\
    Hyperbaton & 80.40 & 80.40 & 80.00 & 80.40 & \textbf{81.20} & 80.80 \\
    Logical Deduction (Five) & 40.40 & \textbf{41.60} & 40.40 & 40.80 & 40.80 & 41.20 \\
    Logical Deduction (Seven) & 31.20 & 30.00 & 30.40 & \textbf{31.20} & 30.00 & 30.40 \\
    Logical Deduction (Three) & 67.20 & 67.20 & 67.20 & 67.20 & \textbf{67.60} & 67.20 \\
    Movie Recommendation & 66.80 & 66.40 & 66.40 & 66.40 & \textbf{67.20} & 66.40 \\
    Multi-Step Arithmetic & 76.00 & 77.60 & 78.00 & 76.80 & \textbf{78.40} & 77.60 \\
    Navigate & 84.40 & 84.40 & 84.40 & 84.00 & \textbf{84.80} & 84.00 \\
    Object Counting & 78.40 & 78.80 & 78.80 & \textbf{79.20} & 76.80 & \textbf{79.20} \\
    Penguins in a Table & 63.70 & \textbf{65.07} & 64.38 & \textbf{65.07} & 63.70 & \textbf{65.07} \\
    Reasoning about Colored Objects & 72.80 & 71.60 & \textbf{74.00} & 72.80 & 73.20 & 72.40 \\
    Ruins Names & 48.00 & 47.60 & 47.20 & 47.20 & \textbf{49.60} & 47.20 \\
    Salient Translation Error & 30.80 & 40.00 & 40.00 & 40.40 & 35.60 & \textbf{40.80} \\
    Snarks & 60.11 & 60.11 & 60.11 & 59.55 & \textbf{61.24} & 58.99 \\
    Sports Understanding & 73.60 & 73.60 & 73.60 & 73.20 & \textbf{74.40} & 74.00 \\
    Temporal Sequences & \textbf{32.40} & 30.40 & 30.80 & 30.80 & 31.60 & 30.80 \\
    Tracking Shuffled Objects (Five) & 36.80 & 38.00 & \textbf{38.80} & 37.60 & 38.40 & 37.60 \\
    Tracking Shuffled Objects (Seven) & 27.60 & 26.80 & 27.20 & \textbf{29.60} & 28.00 & 28.00 \\
    Tracking Shuffled Objects (Three) & 60.80 & 60.80 & 61.20 & 61.20 & \textbf{61.60} & 61.20 \\
    Web of Lies & 99.60 & 99.60 & 99.60 & 99.60 & 99.60 & 99.60 \\
    Word Sorting & 14.40 & 14.80 & \textbf{15.20} & 14.40 & \textbf{15.20} & 14.80 \\
    \midrule
    \textbf{Average Score} & 55.28 & 55.87 & \textbf{55.99} & 55.85 & 55.65 & 55.79 \\
    \bottomrule
  \end{tabular}
  }
\end{table}

\begin{table}[H]
  \caption{\textbf{Full BBH Results on Qwen3-4B.}}
  \label{tab:bbh_4b}
  \centering
  \renewcommand{\arraystretch}{1.15}
  \resizebox{0.95\textwidth}{!}{
  \begin{tabular}{l|c|ccccc}
    \toprule
    \textbf{Subtask} & \textbf{Initial Model} & \textbf{GRPO} & \textbf{LASER} & \textbf{W-REINFORCE} & \textbf{EntroPIC} & \textbf{T2T (Ours)} \\
    \midrule
    Boolean Expressions & 95.70 & \textbf{97.60} & \textbf{97.60} & 96.80 & 97.20 & 96.00 \\
    Causal Judgement & 59.70 & 58.82 & 59.89 & 55.08 & \textbf{61.50} & 56.68 \\
    Date Understanding & 81.60 & 83.60 & 83.20 & 83.20 & \textbf{84.00} & 83.60 \\
    Disambiguation QA & 58.50 & 60.80 & 60.40 & 62.00 & 59.20 & \textbf{62.40} \\
    Dyck Languages & 38.50 & \textbf{39.60} & 38.80 & 39.20 & 36.80 & \textbf{39.60} \\
    Formal Fallacies & 75.00 & 73.60 & 72.80 & 74.00 & \textbf{78.40} & 72.80 \\
    Geometric Shapes & 51.50 & \textbf{54.40} & 52.80 & 53.20 & 46.80 & 51.20 \\
    Hyperbaton & 91.50 & 90.80 & \textbf{91.60} & 90.80 & 91.20 & \textbf{91.60} \\
    Logical Deduction (Five) & 67.50 & \textbf{71.60} & 68.80 & 68.40 & 70.40 & 69.60 \\
    Logical Deduction (Seven) & \textbf{37.80} & 27.60 & 30.80 & 28.40 & 31.60 & 27.20 \\
    Logical Deduction (Three) & 95.50 & 96.00 & 96.00 & 96.00 & \textbf{96.80} & 96.40 \\
    Movie Recommendation & 67.50 & 68.80 & 69.20 & 69.20 & \textbf{69.60} & 68.80 \\
    Multi-Step Arithmetic & 98.20 & 98.40 & 98.80 & 98.80 & \textbf{99.60} & 98.40 \\
    Navigate & 90.50 & 92.80 & 92.80 & 92.80 & \textbf{93.60} & 92.80 \\
    Object Counting & 84.50 & \textbf{86.00} & \textbf{86.00} & \textbf{86.00} & 84.80 & 85.60 \\
    Penguins in a Table & 90.80 & 94.52 & 93.15 & 94.52 & 91.78 & \textbf{95.21} \\
    Reasoning about Colored Objects & 94.80 & 96.00 & 96.00 & 96.40 & \textbf{96.80} & 96.40 \\
    Ruins Names & 68.00 & 70.00 & \textbf{72.40} & 68.00 & 69.60 & 66.80 \\
    Salient Translation Error & 60.80 & 62.80 & 61.60 & 62.00 & \textbf{63.60} & 62.40 \\
    Snarks & 69.50 & \textbf{69.66} & 66.85 & 68.54 & 67.98 & \textbf{69.66} \\
    Sports Understanding & 81.80 & 84.40 & 84.40 & 84.40 & 83.60 & \textbf{84.80} \\
    Temporal Sequences & 84.00 & 86.00 & 86.00 & \textbf{86.40} & \textbf{86.40} & 86.00 \\
    Tracking Shuffled Objects (Five) & 95.50 & 96.40 & \textbf{97.60} & \textbf{97.60} & \textbf{97.60} & 97.20 \\
    Tracking Shuffled Objects (Seven) & 89.00 & 92.80 & 92.80 & \textbf{93.20} & 89.60 & \textbf{93.20} \\
    Tracking Shuffled Objects (Three) & 85.50 & 91.60 & 91.20 & \textbf{94.40} & 75.60 & 92.40 \\
    Web of Lies & \textbf{100.00}& \textbf{100.00} & \textbf{100.00} & \textbf{100.00} & \textbf{100.00} & \textbf{100.00} \\
    Word Sorting & 35.80 & 37.20 & 36.40 & 38.00 & 37.60 & \textbf{39.60} \\
    \midrule
    \textbf{Average Score} & 79.33 & \textbf{80.82} & 80.73 & 80.70 & 80.05 & 80.57 \\
    \bottomrule
  \end{tabular}
  }
\end{table}

 \section{Case Study}
\label{app:case_studies}

We present detailed comparisons between the baseline GRPO and our method to illustrate the adaptive nature of the Thickening-to-Thinning (T2T) mechanism. The following visualizations display two contrasting scenarios: (1) a \textbf{Thickening Case} on a hard trigonometric problem, where our method is incentivized to explore a rigorous derivation path—solving a system of linear equations—to correct the baseline's failure; and (2) a \textbf{Thinning Case} on a simple arithmetic problem, where our method effectively prunes the baseline's redundant conversational fillers to achieve inference efficiency without compromising accuracy.

\begin{center}
\bfseries Case 1: Thickening Case
\end{center}

\begin{tcolorbox}[
  colback=blue!5,
  colframe=blue!40!black,
  title=\textbf{Prompt},
  breakable
]
There exist constants \(a\), \(b\), \(c\), and \(d\) such that
\[
(\sin x)^7 = a \sin 7x + b \sin 5x + c \sin 3x + d \sin x
\]
for all angles \(x\). Find \(d\). Let's think step by step and output the final answer in boxed format.

\end{tcolorbox}

\begin{tcolorbox}[
  colback=red!5,
  colframe=red!40!black,
  title=\textbf{GRPO},
  breakable
]
We are given the identity:
\[
(\sin x)^7 = a \sin 7x + b \sin 5x + c \sin 3x + d \sin x
\]
for all angles \( x \), and we are to find the constant \( d \).

\medskip
\noindent\textbf{Step 1: Use the identity for \( \sin^7 x \)}

We can use the \textbf{multiple-angle identity} or \textbf{power-reduction formulas} to express
\( \sin^7 x \) in terms of sines of multiple angles.
Alternatively, we can use the \textbf{complex exponential form} of sine:
\[
\sin x = \frac{e^{ix} - e^{-ix}}{2i}.
\]

So:
\[
\sin^7 x = \left( \frac{e^{ix} - e^{-ix}}{2i} \right)^7.
\]

This will expand into a sum of terms involving \( e^{ikx} \), and we can then take the imaginary
part to get a sum of sines. But this might be tedious. Instead, we can use a known identity or
the \textbf{orthogonality of sine functions}.

\medskip
\noindent\textbf{Step 2: Use the identity for \( \sin^7 x \)}

We can use the identity:
\[
\sin^7 x
= \frac{7}{8} \sin x
- \frac{7}{8} \sin 3x
+ \frac{7}{8} \sin 5x
- \frac{1}{8} \sin 7x.
\]

This is a known identity derived from the \textbf{multiple-angle expansions} of sine powers.
Let us verify this identity by comparing coefficients.

\medskip
\noindent\textbf{Step 3: Compare with the given expression}

We are given:
\[
(\sin x)^7 = a \sin 7x + b \sin 5x + c \sin 3x + d \sin x.
\]

From the identity above:
\[
\sin^7 x
= -\frac{1}{8} \sin 7x
+ \frac{7}{8} \sin 5x
- \frac{7}{8} \sin 3x
+ \frac{7}{8} \sin x.
\]

So matching terms:
\begin{itemize}
    \item \( a = -\frac{1}{8} \)
    \item \( b = \frac{7}{8} \)
    \item \( c = -\frac{7}{8} \)
    \item \( d = \frac{7}{8} \)
\end{itemize}

\medskip
\noindent\textbf{Final Answer:}
\[
\boxed{\frac{7}{8}}
\]

\end{tcolorbox}

\begin{tcolorbox}[
  colback=green!5,
  colframe=green!40!black,
  title=\textbf{T2T},
  breakable
]
We are given the identity:
\[
(\sin x)^7 = a \sin 7x + b \sin 5x + c \sin 3x + d \sin x
\]
for all angles \( x \), and we are to find the constant \( d \).

\medskip
\noindent\textbf{Step 1: Use the identity for \( \sin^7 x \)}

We can use the \textbf{multiple-angle identities} or \textbf{power-reduction formulas} to express
\( \sin^7 x \) in terms of sines of multiple angles.
Alternatively, we can use the \textbf{complex exponential form} of sine:
\[
\sin x = \frac{e^{ix} - e^{-ix}}{2i}.
\]

So:
\[
\sin^7 x = \left( \frac{e^{ix} - e^{-ix}}{2i} \right)^7.
\]

This will expand into a sum of terms involving \( e^{ikx} \), and we can then take the imaginary
part to get back to sines. But this might be tedious. Instead, we can use a known identity or
use symmetry and evaluate the equation at specific values of \( x \) to find the coefficients.

\medskip
\noindent\textbf{Step 2: Use symmetry and evaluate at specific values of \( x \)}

Let us consider the original equation:
\[
(\sin x)^7 = a \sin 7x + b \sin 5x + c \sin 3x + d \sin x.
\]

We can plug in specific values of \( x \) to get equations involving \( a, b, c, d \).

\medskip
\noindent\textbf{Case 1: \( x = 0 \)}

\begin{itemize}
    \item \( \sin 0 = 0 \), so LHS \( = 0^7 = 0 \)
    \item RHS \( = a \sin 0 + b \sin 0 + c \sin 0 + d \sin 0 = 0 \)
\end{itemize}

So this gives no new information.

\medskip
\noindent\textbf{Case 2: \( x = \frac{\pi}{2} \)}

\begin{itemize}
    \item \( \sin \frac{\pi}{2} = 1 \), so LHS \( = 1^7 = 1 \)
\end{itemize}

The RHS is:
\[
a \sin \frac{7\pi}{2}
+ b \sin \frac{5\pi}{2}
+ c \sin \frac{3\pi}{2}
+ d \sin \frac{\pi}{2}.
\]

We compute each term:
\begin{itemize}
    \item \( \sin \frac{7\pi}{2} = \sin\!\left(3\pi + \frac{\pi}{2}\right) = -1 \)
    \item \( \sin \frac{5\pi}{2} = \sin\!\left(2\pi + \frac{\pi}{2}\right) = 1 \)
    \item \( \sin \frac{3\pi}{2} = -1 \)
    \item \( \sin \frac{\pi}{2} = 1 \)
\end{itemize}

So:
\[
\text{RHS} = -a + b - c + d.
\]

Setting RHS equal to LHS gives:
\[
-a + b - c + d = 1 \tag{1}
\]

\medskip
\noindent\textbf{Case 3: \( x = \frac{\pi}{6} \)}

\begin{itemize}
    \item \( \sin \frac{\pi}{6} = \frac{1}{2} \), so LHS \( = \left(\frac{1}{2}\right)^7 = \frac{1}{128} \)
\end{itemize}

The RHS is:
\[
a \sin \frac{7\pi}{6}
+ b \sin \frac{5\pi}{6}
+ c \sin \frac{3\pi}{6}
+ d \sin \frac{\pi}{6}.
\]

Compute each sine:
\begin{itemize}
    \item \( \sin \frac{7\pi}{6} = \sin\!\left(\pi + \frac{\pi}{6}\right) = -\frac{1}{2} \)
    \item \( \sin \frac{5\pi}{6} = \sin\!\left(\pi - \frac{\pi}{6}\right) = \frac{1}{2} \)
    \item \( \sin \frac{3\pi}{6} = \sin \frac{\pi}{2} = 1 \)
    \item \( \sin \frac{\pi}{6} = \frac{1}{2} \)
\end{itemize}

So:
\[
\text{RHS}
= -\frac{a}{2} + \frac{b}{2} + c + \frac{d}{2}.
\]

Setting equal to LHS:
\[
-\frac{a}{2} + \frac{b}{2} + c + \frac{d}{2}
= \frac{1}{128} \tag{2}
\]

\medskip
\noindent\textbf{Case 4: \( x = \frac{\pi}{3} \)}

\begin{itemize}
    \item \( \sin \frac{\pi}{3} = \frac{\sqrt{3}}{2} \)
    \item LHS \( = \left(\frac{\sqrt{3}}{2}\right)^7 = \frac{27\sqrt{3}}{128} \)
\end{itemize}

The RHS is:
\[
a \sin \frac{7\pi}{3}
+ b \sin \frac{5\pi}{3}
+ c \sin \pi
+ d \sin \frac{\pi}{3}.
\]

Compute each sine:
\begin{itemize}
    \item \( \sin \frac{7\pi}{3} = \sin\!\left(2\pi + \frac{\pi}{3}\right) = \frac{\sqrt{3}}{2} \)
    \item \( \sin \frac{5\pi}{3} = -\frac{\sqrt{3}}{2} \)
    \item \( \sin \pi = 0 \)
    \item \( \sin \frac{\pi}{3} = \frac{\sqrt{3}}{2} \)
\end{itemize}

So:
\[
\text{RHS}
= \frac{\sqrt{3}}{2}(a - b + d).
\]

Setting equal to LHS:
\[
\frac{\sqrt{3}}{2}(a - b + d) = \frac{27\sqrt{3}}{128}
\]
\[
\Rightarrow a - b + d = \frac{27}{64} \tag{3}
\]

\medskip
Now we have three equations:
\begin{itemize}
    \item (1) \( -a + b - c + d = 1 \)
    \item (2) \( -\frac{a}{2} + \frac{b}{2} + c + \frac{d}{2} = \frac{1}{128} \)
    \item (3) \( a - b + d = \frac{27}{64} \)
\end{itemize}

We are to find \( d \).

From (3):
\[
a = b - d + \frac{27}{64} \tag{4}
\]

Substitute into (1):
\[
2d - c = \frac{91}{64} \tag{5}
\]

Substitute into (2):
\[
d + c = \frac{7}{32} \tag{6}
\]

Adding (5) and (6):
\[
3d = \frac{105}{64}
\Rightarrow d = \frac{35}{64}.
\]

\medskip
\noindent\textbf{Final Answer:}
\[
\boxed{\frac{35}{64}}
\]

\end{tcolorbox}

\begin{center}
\bfseries Case 2: Thinning Case
\end{center}

\begin{tcolorbox}[
  colback=blue!5,
  colframe=blue!40!black,
  title=\textbf{Prompt},
  breakable
]
How many positive whole-number divisors does 196 have?  
Let's think step by step and output the final answer in boxed format.
\end{tcolorbox}

\begin{tcolorbox}[
  colback=green!5,
  colframe=green!40!black,
  title=\textbf{GRPO},
  breakable
]
We are asked to find how many \textbf{positive whole-number divisors} the number \textbf{196} has.

\medskip
\noindent\textbf{Step 1: Prime Factorization}

We start by factoring 196 into its prime factors.
\[
196 = 2 \times 98 = 2 \times 2 \times 49 = 2^2 \times 7^2
\]

So, the \textbf{prime factorization} of 196 is:
\[
196 = 2^2 \times 7^2
\]

\medskip
\noindent\textbf{Step 2: Use the Divisor Count Formula}

If a number \( n \) has the prime factorization:
\[
n = p_1^{e_1} \times p_2^{e_2} \times \dots \times p_k^{e_k},
\]
then the number of \textbf{positive divisors} of \( n \) is:
\[
(e_1 + 1)(e_2 + 1)\dots(e_k + 1).
\]

For 196, the exponents are:
\begin{itemize}
    \item \( 2^2 \) → exponent = 2
    \item \( 7^2 \) → exponent = 2
\end{itemize}

So the number of positive divisors is:
\[
(2 + 1)(2 + 1) = 3 \times 3 = 9.
\]

\medskip
\noindent\textbf{Final Answer:}
\[
\boxed{9}
\]

So, \textbf{196 has 9 positive whole-number divisors}.

\end{tcolorbox}

\begin{tcolorbox}[
  colback=green!5,
  colframe=green!40!black,
  title=\textbf{T2T},
  breakable
]
To find the number of \textbf{positive whole-number divisors} of 196, we start by finding its \textbf{prime factorization}.

\medskip
\noindent\textbf{Step 1: Prime Factorization of 196}

\[
196 = 2 \times 98 = 2 \times 2 \times 49 = 2^2 \times 7^2
\]

So, the prime factorization is:
\[
196 = 2^2 \times 7^2
\]

\medskip
\noindent\textbf{Step 2: Use the Formula for Number of Divisors}

If a number
\[
n = p_1^{e_1} \times p_2^{e_2} \times \dots \times p_k^{e_k},
\]
then the number of positive divisors is:
\[
(e_1 + 1)(e_2 + 1)\dots(e_k + 1).
\]

For 196:
\begin{itemize}
    \item Exponent of 2 is 2
    \item Exponent of 7 is 2
\end{itemize}

So, the number of divisors is:
\[
(2 + 1)(2 + 1) = 3 \times 3 = 9
\]

\medskip
\noindent\textbf{Final Answer:}
\[
\boxed{9}
\]
\end{tcolorbox}

\section{Limitations}
\label{app:limitations}
This work has several limitations that leave room for future improvement.

First, we estimate problem difficulty solely based on the \textbf{on-policy pass rate} during training. Such a single-point online estimation is susceptible to sampling randomness and training fluctuations, which may lead to inaccurate difficulty calibration across benchmarks. To mitigate this issue, future work can adopt statistical smoothing strategies such as Exponential Moving Average (EMA) to stabilize pass-rate statistics, filter transient noise, and yield more reliable difficulty scores for adaptive T2T reward modulation.

Second, our current implementation and experiments are constrained to \textbf{binary-verifiable RLVR settings}. Nevertheless, the Thickening-to-Thinning (T2T) principle originates from general human cognitive learning dynamics and is not inherently bounded to rigid verifiable tasks. Extending T2T to open-domain reasoning remains a promising direction. In non-verifiable scenarios such as open-domain QA and commonsense reasoning, explicit binary verification is no longer available; task difficulty and model mastery can instead be proxied by soft reward signals, including reward model scores or self-consistency entropy across multiple reasoning paths. The core T2T mechanism still holds: we encourage expansive, thickened reasoning for low-confidence or low-reward queries, and promote concise, thinned reasoning for high-confidence, well-mastered ones. Beyond open-domain reasoning, T2T is also compatible with modern on-policy and self-distillation paradigms. The T2T-adjusted reward signal can be incorporated into reverse KL divergence distillation objectives, guiding student models to learn dynamically optimized reasoning trajectories adapted to implicit task difficulty. We leave these broader extensions to future research.

\section{Hardware Setup}
\label{app:hardware_setup}
All experiments are conducted on a Linux-based computing cluster, with GPU configurations tailored to different model scales:
\begin{enumerate}
    \item \textbf{1.5B-scale models}: Trained on NVIDIA RTX3090 GPUs.
    \item \textbf{3B–4B-scale models}: Trained on NVIDIA A100-PCIE-40GB GPUs.
    \item \textbf{14B-scale models}: Trained on NVIDIA A800-PCIE-80GB GPUs.
\end{enumerate}



\end{document}